\documentclass[letterpaper, 10 pt, conference]{IEEEconf}
\IEEEoverridecommandlockouts

\usepackage{cite}
\usepackage{amsmath,amssymb,amsfonts}
\usepackage{graphicx}
\usepackage{textcomp}
\usepackage{xcolor}

\usepackage{algorithm}
\usepackage{algpseudocode}
\usepackage{subcaption}
\usepackage[font=small]{caption}
\usepackage{booktabs}
\usepackage{multirow}

\usepackage{hyperref}
\hypersetup{
    colorlinks=true,
    linkcolor=black,
    urlcolor=black
}

\renewcommand{\QED}{\QEDopen}

\usepackage{amsthm}

\newtheorem{problem}{Problem}

\theoremstyle{definition}
\newtheorem{definition}{Definition}

\theoremstyle{plain}
\newtheorem{lemma}{Lemma}

\theoremstyle{plain}
\newtheorem{proposition}{Proposition}

\theoremstyle{remark}
\newtheorem{remark}{Remark}

\def\BibTeX{{\rm B\kern-.05em{\sc i\kern-.025em b}\kern-.08em
    T\kern-.1667em\lower.7ex\hbox{E}\kern-.125emX}}
    
\begin{document}

\title{\huge Efficient Multi-Objective Planning with Weighted Maximization using Large Neighbourhood Search}

\author{Krishna Kalavadia, Shamak Dutta, Yash Vardhan Pant, and Stephen L. Smith%
\thanks{This research was undertaken, in part, thanks to funding from the Canada Research Chairs Program, and in part by the Natural Sciences and Engineering Research Council of Canada (NSERC). 
\protect\\ \indent The authors are with the Department of Electrical and Computer Engineering, University of Waterloo, Waterloo, Canada (e-mails: \{kkalavad, shamak.dutta, yash.pant, stephen.smith\}@uwaterloo.ca).%
\protect\\ \indent
Code available at: \url{https://github.com/CL2-UWaterloo/WM-LNS}.}}

\maketitle

\begin{abstract}
Autonomous navigation often requires the simultaneous optimization of multiple objectives. The most common approach scalarizes these into a single cost function using a weighted sum, but this method is unable to find all possible trade-offs and can therefore miss critical solutions. An alternative, the weighted maximum of objectives, can find all Pareto-optimal solutions, including those in non-convex regions of the trade-off space that weighted sum methods cannot find. However, the increased computational complexity of finding weighted maximum solutions in the discrete domain has limited its practical use. To address this challenge, we propose a novel search algorithm based on the Large Neighbourhood Search framework that efficiently solves the weighted maximum planning problem. Through extensive simulations, we demonstrate that our algorithm achieves comparable solution quality to existing weighted maximum planners with a runtime improvement of 1-2 orders of magnitude, making it a viable option for autonomous navigation.
\end{abstract}
\hfill 


\section{Introduction}
In many scenarios, autonomous robots optimize multiple, possibly competing, objectives such as path length, obstacle clearance, and energy efficiency. Solving a multi-objective optimization problem does not yield a single unique solution but a set of 
\textit{Pareto-optimal} solutions where for each solution, its individual objectives cannot be improved without sacrificing performance in another \cite{branke_multiobjective_2008}. This set of solutions is known as the \textit{Pareto front}. One approach to describe the trade-offs between objective functions is to scalarize them, providing a single cost function that can be optimized.

A common scalarization method is to define the cost function as a \textit{weighted sum} (WS) of the objectives, where the weights are tuning parameters chosen by an external decision maker. Although widely used, the WS method is unable to return solutions that lie in the non-convex region of the Pareto front, regardless of the weights chosen \cite{branke_multiobjective_2008}. An alternative approach is to take the \textit{weighted maximum} (WM) of objectives, also known as Chebyshev scalarization. This method is able to find all Pareto-optimal solutions, including those in non-convex regions \cite{branke_multiobjective_2008}.

Figure \ref{fig:moo_pipeline} illustrates the scalarization approach to solving multi-objective path planning problems. Consider a mobile robot navigating an indoor environment. An external decision maker defines the importance of various objectives, such as path length and minimum distance to obstacles (obstacle closeness), and the planner returns a Pareto-optimal path. The WS scalarization cannot find a path between the obstacles because it corresponds to a trade-off that lies in the non-convex region of the Pareto front. This highlights the benefits of the WM scalarization, as important trade-offs can be missed by the WS. 

\begin{figure}[t]
    \centering
    \includegraphics[height=5.45cm]{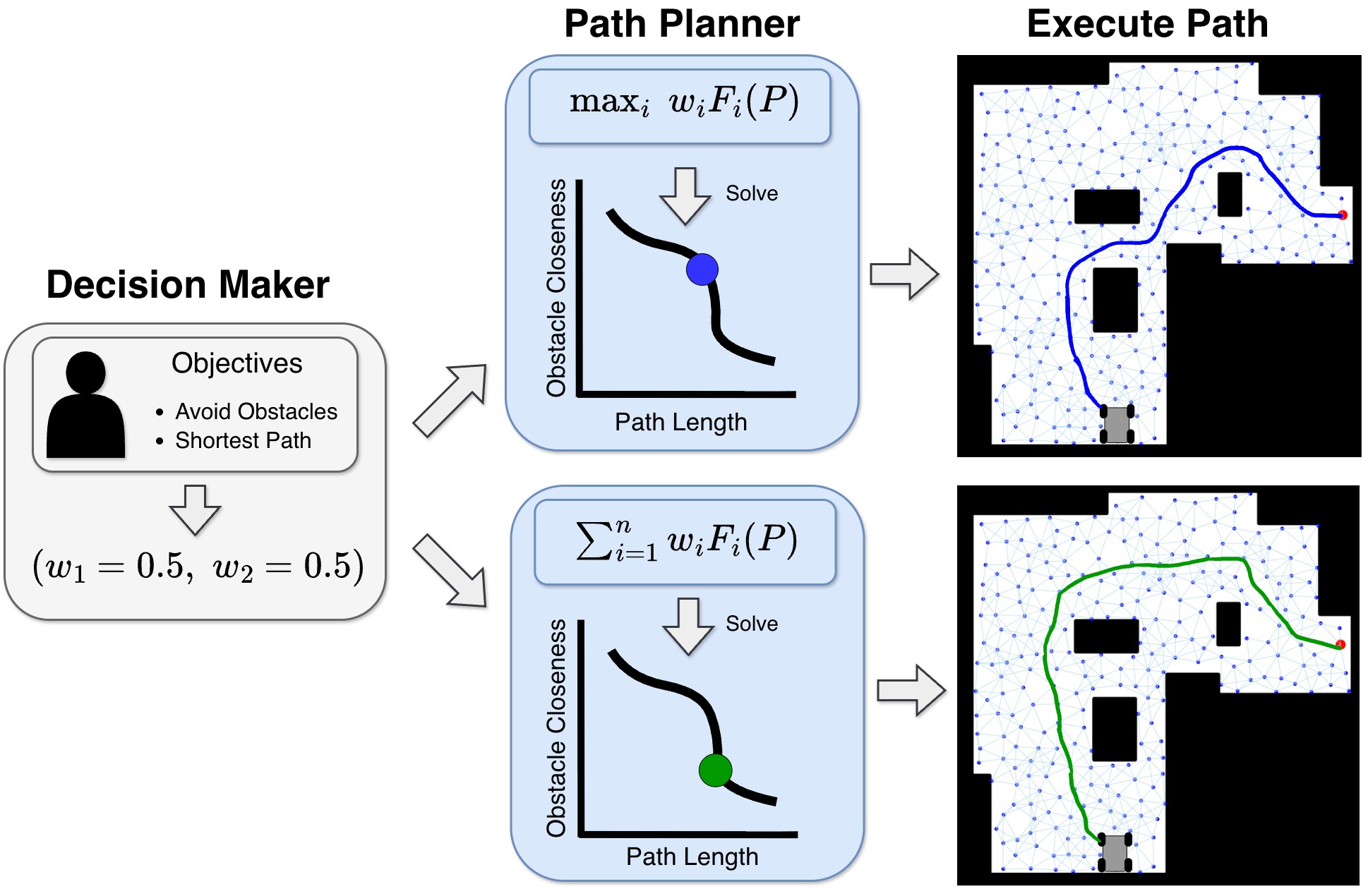}
    \caption{\small Solving multi-objective path planning problems via scalarization. An external decision maker specifies objective preferences as weights $w_1, w_2$. The path planner solves the scalarized problem and returns a path for the robot to execute. The WS cannot return the solution matching the desired objective preferences because it lies in the non-convex region of the Pareto front.}
    \label{fig:moo_pipeline}
    \vspace{-1mm}
\end{figure}

The authors of \cite{wilde_scalarizing_2024} present a formal framework for WM-based planning and demonstrate its potential in both continuous and discrete domains. However, in discrete space, the WM formulation makes the problem NP-hard \cite{wilde_scalarizing_2024}, while the WS problem remains solvable in polynomial time. 

To address this gap, we propose an efficient and novel heuristic algorithm that leverages Large Neighbourhood Search. The proposed algorithm, WM-LNS, reduces the computational time required to solve discrete WM planning problems by 1-2 orders of magnitude. We demonstrate the effectiveness of WM-LNS through extensive simulation studies and aim to remove the computational barrier that limits the wider application of WM methods for robot planning.

\subsection{Contributions} First, we study the limitations of the WS scalarization for multi-objective optimization. We provide a worst-case bound for approximating the WM with the WS and prove NP-hardness of finding the best WS approximation. Second, we propose a novel LNS algorithm to solve the WM planning problem that exploits the property that the Pareto fronts of WM subproblems can have convex geometry and are therefore solvable by the WS. Additionally, through an extensive set of experiments, we demonstrate that the proposed algorithm can find high-quality WM solutions with a runtime improvement of 1-2 orders of magnitude across various planning tasks. 

\subsection{Related Work} \textit{Multi-Objective Robot Planning:} The most common approach to multi-objective planning is to scalarize the objective functions as a weighted sum. This approach is used across various robotic applications from autonomous driving \cite{yang_automated_2020, lim_hybrid_2021}, mobile robotics \cite{lu_motion_2020, lindqvist_exploration-rrt_2021, naazare_online_2022}, and modelling user preferences in human-robot interaction \cite{sadigh-rss17, wilde_improving_2020}.

An alternative approach to scalarization is to compute a set of solutions that covers the entire Pareto front, including those in the non-convex region. The \(\epsilon\)-constraint method involves selecting one objective to optimize and converting the remaining objectives to constraints with some upper-bound on their value \cite{branke_multiobjective_2008}. The Adaptive Weighted Sum Method expands on the standard WS method by iteratively adding equality constraints, forcing new solutions to cover non-convex and sparsely sampled regions of the Pareto front \cite{kim_adaptive_2005}, \cite{kim_adaptive_2006}. A more popular method for computing a set of solutions covering the Pareto front is to leverage evolutionary algorithms \cite{ren_multi-objective_2019, zheng_gaussian_2024}, which maintain a population of non-dominated solutions. However, these methods address a different class of robotics problems than our paper, as they approximate the full set of trade-offs between objectives. They do not provide a mechanism for a decision maker to explicitly target a particular preferred solution \cite{10599820}.

Addressing the limitations of the WS is the WM scalarization. The authors of \cite{wilde_scalarizing_2024} provide a framework for solving multi-objective planning problems with WM cost functions. In continuous space, the WM method is not fundamentally harder to solve than the WS method. However, in discrete space, the WM problem is NP-hard while the WS problem is solvable in polynomial time \cite{wilde_scalarizing_2024}. Our paper addresses this increased complexity and proposes a novel approach to efficiently solve WM planning problems in discrete space. 

\textit{Large Neighbourhood Search:} Our paper is also built upon Large Neighbourhood Search, a stochastic meta-heuristic framework used in combinatorial optimization. LNS operates by iteratively destroying and repairing a solution until a stop condition is met \cite{shaw_using_1998}. The authors of \cite{ropke_adaptive_2006} extended the framework, introducing an adaptive LNS which has been applied to various problems such as team orienteering  \cite{kim2013augmented}, vehicle routing \cite{pisinger_general_2007, azi2014adaptive}, and multi-agent path finding \cite{li_anytime_2021}. Our work provides the first adaptive LNS algorithm for solving multi-objective robot planning problems. 

\section{Problem Formulation}
We consider robot navigation in an environment discretized into a connected graph \(G = (V, E)\) with vertex set \(V\) and edge set \(E\). Let \(\mathcal{P}_{s,g}\) be the set of all simple paths starting at vertex \(s\) and ending at a goal vertex \(g\). We look to minimize a set of objectives \(F_1, \dots, F_n\) where each \(F_{i}:\mathcal{P}_{s,g}\to\mathbb{R}_{\geq0}\).  Each edge \(e \in E\) has \(n\) non-negative traversal costs \(\mathbf{f}(e) = (f_1(e), \dots, f_n(e))\). The cost vector of a path \(P\) with \(k\) edges, \(e_1,e_2,\ldots,e_k\), is given by \(\mathbf{F}(P ) = (F_1(P), \dots, F_n(P))\) where \(F_i(P) = f_i(e_1)+ f_i(e_2) +\dots +f_i(e_k) \). The set of all cost vectors is defined as \(\mathcal{F}_{s,g} := \{\mathbf{F}(P) \in \mathbb{R}^n_{\geq 0} \;|\; P \in \mathcal{P}_{s,g} \}\).

\begin{definition}[Pareto-optimality]\label{def:pareto}
Consider two solution paths \(P_1\) and \(P_2\). We say \(P_1\) \textit{dominates} \(P_2\) if for all \(i\in\{1,\dots,n\}\), \(F_{i}(P_1)\le F_{i}(P_2)\) and there exists some \(j\in\{1,\dots,n\}\) where \(F_{j}(P_1)<F_{j}(P_2)\). A solution \(P^*\) is \textit{Pareto-optimal} if there exists no other solution \(P'\in \mathcal{P}_{s,g}\) such that \(P'\) dominates \(P^*\). The set of all Pareto-optimal solutions is the \textit{Pareto front}, and an algorithm is \textit{Pareto-complete} if it can discover the entirety of the Pareto front.
\end{definition}

A Pareto-optimal solution \(P^* \in \mathcal{P}_{s,g}\) is one that minimizes the set of objectives  
\begin{equation}
  \min_{P \in \mathcal{P}_{s,g}} \; \{F_{1}(P),F_{2}(P),\,\dots,\,F_{n}(P)\}. \label{eq:moo_problem}
\end{equation}

Solving \eqref{eq:moo_problem} results in a set of Pareto-optimal solutions. Our approach is to consider a scalarization of \eqref{eq:moo_problem} using a set of non-negative weights specified by an external decision maker. Given \(\textbf{w} =\{ w_1,\dots,w_n\} \in \mathcal{W}\), where \(\mathcal{W} \;=\; \{\, \mathbf{w}\in \mathbb{R}_{\ge 0}^n \; | \; \sum_{i=1}^{n} w_i = 1 \}\) denotes the weight space (equivalent to the unit simplex), the most common cost scalarization is the weighted sum (WS)
\begin{equation}
    \text{WS}(P) := \sum_{i = 1}^{n} w_i F_i(P). \label{eq:ws_cost}
\end{equation}
For a fixed weight, the optimal WS solution to \eqref{eq:moo_problem} is denoted \( P_{\text{WS}} = \arg\min_{P \in \mathcal{P}_{s,g}} \text{WS}(P)\) and the set of WS solutions optimal for \textit{some} weight is denoted \( \mathcal{P}_{\text{WS}} \). 
An alternative approach is the weighted maximum (WM) cost
\begin{equation}
    \text{WM}(P) := \max_{i} \;  w_i F_i(P). \label{eq:wm_cost}
\end{equation}
Similarly, for a fixed weight, the optimal WM solution is \( P_{\text{WM}} = \arg\min_{P \in \mathcal{P}_{s,g}} \text{WM}(P)\) and the set of WM solutions optimal for \textit{some} weight is denoted \(\mathcal{P}_{\text{WM}}\).

The WM scalarization is capable of returning a richer set of solutions describing the various trade-offs between objective functions \cite{branke_multiobjective_2008}.
Therefore, in this paper, we focus on solving planning problems with a WM cost function.

\begin{problem} \label{def:problem2}
Given weights \(\textbf{w} \in \mathcal{W}\) and a connected graph \(G\) with start and goal vertices \(s\) and \(g\) respectively, find a path \(P^* \in \mathcal{P}_{s,g}\) that solves
\[
\min_{P \in \mathcal{P}_{s,g}} \; \max_{i} \; w_i F_i(P).
\]
\end{problem}

To avoid obtaining \textit{weakly} Pareto-optimal solutions, a tie-breaker term, \(\rho\ \sum_{i = 1}^{n}F_i(P)\), can be added to the cost function, where \(\rho > 0\) is a sufficiently small constant \cite{branke_multiobjective_2008, wilde_scalarizing_2024}. The WM formulation is Pareto-complete, that is, for any Pareto-optimal path \(P^*\), there exists a choice of weights such that \(P^*\) is a solution to Problem \ref{def:problem2} \cite{branke_multiobjective_2008, wilde_scalarizing_2024}. Unfortunately, solving this problem directly is NP-hard \cite{wilde_scalarizing_2024}.

\section{Linking Weighted Maximization and Summation}
Given that the WM problem is NP-hard while the WS problem can be solved in polynomial time, a natural question is how well the WS scalarization can approximate the WM scalarization. In Section~\ref{subsec:limitations}, we explore the fundamental limits of the WS scalarization by i) providing a worst-case bound for using the WS to approximate the WM scalarization and ii) proving that finding the set of weights that achieves the best approximation is NP-hard. However, in Section~\ref{subsec:convex-subproblem}, we motivate our solution approach by demonstrating how to solve the WM problem by solving a \emph{sequence of sub-problems} using the WS scalarization.

\subsection{Limitations of the Weighted Sum} \label{subsec:limitations}

Solving \eqref{eq:moo_problem} with the WS scalarization is geometrically equivalent to finding a point of contact on the boundary of the convex hull of \(\mathcal{F}_{s,g}\) with a supporting hyperplane defined by the normal vector \((w_1, \dots, w_n) \in \mathcal{W}\) that minimizes the linear function defined by the weights. 

\begin{lemma}\label{lem:lemma1}
Let \( \mathcal{P}_{\text{SPO}} \subseteq \mathcal{P}_{s,g}\) be the set of all paths whose cost vectors are non-dominated and lie on the boundary of the convex hull of \(\mathcal{F}_{s,g}\). Then, \(\mathcal{P}_{\text{WS}} = \mathcal{P}_{\text{SPO}}\). 
\end{lemma} 
The proof follows from the discussion provided in Sections 4.7.3 and 4.7.4 in \cite{boyd2004convex}. If the Pareto front is non-convex, there exist solutions that do not have a supporting hyperplane, and therefore cannot be found by the WS scalarization. 

\begin{proposition}[Proposition 1, \cite{wilde_scalarizing_2024}]\label{lem:lemma2}
Let \(\mathcal{P'} \subseteq \mathcal{P}_{s,g}\) be the set of all Pareto-optimal solution paths to \eqref{eq:moo_problem}. If the Pareto front is non-convex, then, \(\mathcal{P}_{\text{SPO}} = \mathcal{P}_{\text{WS}} \;\subset\; \mathcal{P}_{\text{WM}} \;=\; \mathcal{P}' \;\subseteq\; \mathcal{P}_{s,g}\).
\end{proposition}

The proof follows from an established property in the multi-objective optimization literature that the WS cannot find solutions that lie in the non-convex region of the Pareto front, while the WM is Pareto-complete \cite{branke_multiobjective_2008, wilde_scalarizing_2024}. In the case of a non-convex Pareto front, we provide a worst-case approximation bound for using the WS to approximate WM. 

\begin{proposition}[Approximation Bound]\label{prop:prop1} Consider solving the weighted max scalarization and weighted sum scalarization of \eqref{eq:moo_problem} with a fixed \(\mathbf{w} \in \mathcal{W}\), yielding \(P_{\text{WM}}\) and \(P_{\text{WS}}\) respectively. If there are \(n\) objectives, then 
\[
\text{WM}(P_{\text{WS}}) \leq n \cdot \text{WM}(P_{\text{WM}}).
\]
\end{proposition}

\begin{proof}
The WM cost is the largest objective function among \(w_1F_{1}(P_{\text{WM}}),\,\dots,\,w_nF_n({P_{\text{WM}})}\), that is \(w_i F_i(P_{\text{WM}}) \leq \text{WM}(P_{\text{WM}})\). Adding each of these inequalities yields
\[
    \begin{aligned}
        w_1F_{1}(P_{\text{WM}}) \; +\dots + \; w_nF_n(P_{\text{WM}}) &\leq n \cdot \text{WM}(P_{\text{WM}}) \\
        \text{WS}(P_{\text{WM}}) &\leq n \cdot \text{WM}(P_{\text{WM}}).
    \end{aligned}
\]
By the optimality of the WS problem, it follows that
\[
  \text{WS}(P_{\text{WS}}) \leq \text{WS}(P_{\text{WM}}) \leq n \cdot \text{WM}(P_{\text{WM}}).
\]
Finally, observe that for any path \(P\in\mathcal{P}_{s,g}\) we have \(\text{WM}(P) \leq \text{WS}(P)\), since the weights are non-negative and the maximum of the weighted components cannot be greater than the sum. We can now construct our final bound $\text{WM}(P_{\text{WS}}) \leq n \cdot \text{WM}(P_{\text{WM}})$.
\end{proof}

Geometrically, the worst-case \(n\)-factor bound is approached when the entirety of the Pareto front is non-convex, and the WM solution lies in its centre. In this case, the WS solution is forced to the supported endpoints of the Pareto front and therefore by Lemma \ref{lem:lemma1}, the bound cannot be reduced even if one resolves the WS with a different weight vector \(\hat{\textbf{w}}\).

The \(n\)-factor bound provides a worst-case guarantee. However, an open question remains on how the weights for the WS problem can be selected to minimize the approximation error for a given WM solution.

\begin{problem}[Best WS Approximation (BWSA)]\label{def:problem3}
Given some weight vector \(\mathbf{w} \in \mathcal{W}\) used to obtain \(P_{\text{WM}}\), the objective is to find a proxy weight vector \(\hat{\mathbf{w}} \in \mathcal{W}\) that solves
\[
\begin{aligned}
    \min_{\hat{\mathbf{w}} \in \mathcal{W}} & \; \max_{i=1,\ldots,n} \{w_i F_i(P)\}  \\
\text{subject to}\quad & P = \underset{P \in \mathcal{P}_{s,g}}{\mathrm{argmin}} \sum_{i = 1}^{n} \hat{w_i} F_i(P). 
\end{aligned}
\]
Observe that when \(\hat{\mathbf{w}}\) varies, \(P\) varies across \(\mathcal{P}_{\text{WS}}\), which by Lemma \ref{lem:lemma1} is equivalent to \(\mathcal{P}_{\text{SPO}}\). Therefore, we can rewrite the BWSA as 
\[
    \min_{P \in \mathcal{P}_{\text{SPO}}} \; \max_{i=1,\ldots,n} \{w_i F_i(P)\}.
\]
\end{problem}

\begin{proposition}[Hardness of BWSA] The Best WS Approximation Problem is NP-hard.
\end{proposition}
\begin{proof} The weighted max problem stated in Problem \ref{def:problem2} is known to be NP-hard as shown by \cite{wilde_scalarizing_2024}. We provide a proof sketch that reduces Problem \ref{def:problem2} to Problem \ref{def:problem3}.

Consider an arbitrary instance of Problem \ref{def:problem2} \(I = (G, s, g, \mathbf{f}, \mathbf{w})\) with \(m = |E|\) edges. In polynomial time, we transform it to a new instance \(I' = (G, s, g, \mathbf{f'}, \mathbf{w'})\) by augmenting the costs on each edge and the weight vector. We increase the number of objectives from \(n\) to \(n + m\) so that each edge \(e_j \in E\) now has the cost vector
\[
    \mathbf{f}'(e_j)
    = (
    \underbrace{f_1(e_j),\,\ldots,\,f_n(e_j)}_{\text{original }n},
    \underbrace{0,\,\ldots,\,0,\,1,\,0,\,\ldots,\,0}_{\text{\(m\) added components}}),
\]
where the added `1' is at position \(n + j\). The new weight vector used for evaluation is now \(\mathbf{w}' = (w_1, \dots, w_n, 0, \dots, 0)\). In the newly constructed instance, we denote the set of all paths \(\mathcal{P}_{s,g}'\) and the set of all path cost vectors \(\mathcal{F}_{s,g}'\). By adding a new edge indicator cost for each edge in the graph, we ensure two properties. First, this ensures that the cost vector of each path is now a vertex on \(conv(\mathcal{F}_{s,g}')\) since it cannot be written as a convex combination of any other cost vector in \(\mathcal{F}_{s,g}'\). Second, it ensures the cost vector of each path is non-dominated since for any two distinct paths \(Q, P \in \mathcal{P}_{s,g}'\), \(Q \neq P\),  there exists an edge \(e_k\) in \(P\) that does not exist in \(Q\), such that \(f'(P)_{n + k} = 1 > f'(Q)_{n + k} = 0\), and vice versa, therefore neither path is dominated. Lemma \ref{lem:lemma1} shows \(\mathcal{P'}_{WS} = \mathcal{P}_{SPO}'\) and the two properties derived above establish \(\mathcal{P'}_{SPO} = \mathcal{P}_{s,g}'\). Therefore, the BWSA problem on instance \(I'\) requires solving
\begin{equation}
    \min_{P \in \mathcal{P'}_{s,g}} \; \max_{i=1,\ldots,n+m} \{w_i' F_i'(P)\}.
    \label{eq:bwsa_opt}
\end{equation}

We now show that (\ref{eq:bwsa_opt}) is equivalent to Problem \ref{def:problem2}. Since \(\mathbf{w}' = (w_1, \dots, w_n, 0, \dots, 0)\), we can rewrite (\ref{eq:bwsa_opt}) as
\[
    \min_{P \in \mathcal{P}_{s,g}'} \! \ \max ( \!
    \max_{i \in \{1,\ldots,n\}} \! \{ w_i F_i(P) \}, \!\!
    \max_{i \in \{n+1,\ldots,n+m\}} \! \{ 0 \cdot F'_i(P) \}),
\]
By construction, the vertex and edge sets of \(I'\) remain unchanged from \(I\), therefore \(\mathcal{P}_{s,g}'\) = \(\mathcal{P}_{s,g}\). Substituting the domain and eliminating the zero-weighted terms yields:
\begin{equation}
    \min_{P \in \mathcal{P}_{s,g}} \ \max_{i \in \{1,\ldots,n\}} \{ w_i F_i(P) \}. \label{eq:wm_opt}
\end{equation}
Observe that (\ref{eq:wm_opt}) is identical to the cost function of Problem \ref{def:problem2} over the same domain. Therefore, the optimal solution obtained by solving the BWSA problem on instance \(I'\) is equal to the optimal solution for the WM problem on instance \(I\). Since the transformation of instance \(I\) to \(I'\) is polynomial in the size of the input graph and the WM problem is NP-hard, the BWSA problem must be NP-hard.
\end{proof}

\begin{remark}[Hardness for a fixed number of objectives]
    This reduction depends on the size of the input graph. To strengthen this result and show NP-hardness for a fixed number of objectives, the edge perturbation can be replaced by a constant-dimension perturbation.  
\end{remark}

While these results highlight the fundamental limits of the WS scalarization, we explore recovering an unsupported \(P_{\text{WM}}\) solution using the WS on a \emph{sequence of subproblems} instead of optimizing over the entire set \(\mathcal{P}_{s,g}\). This would enable us to retain the Pareto-complete property of the WM scalarization with the efficiency of the WS scalarization.

\subsection{On the Convexity of Subproblem Pareto Fronts} \label{subsec:convex-subproblem}

\begin{figure}[t]
    \centering
    \includegraphics[height=4.25cm]{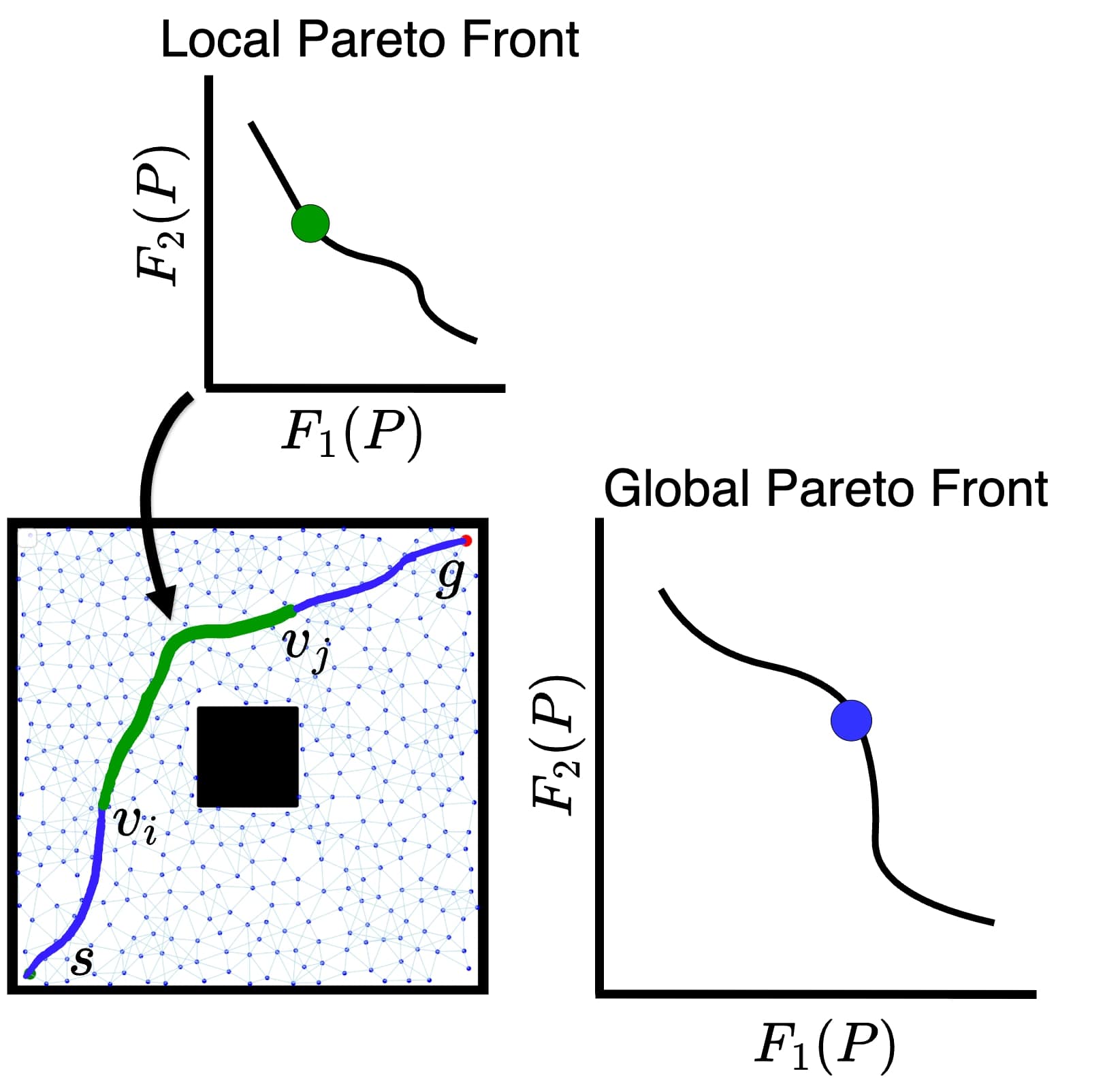}
    \caption{\small An example of the changing geometry of the Pareto front depending on the subproblem considered. The global Pareto front corresponds to optimizing over \(\mathcal{P}_{s, g}\). The local Pareto front corresponds to the problem of optimizing over \(\mathcal{P}_{v_i, v_j}\).}
    \label{fig:local_vs_global_pf}
    \vspace{-2mm}
\end{figure}

Consider a WM solution \(P_{\text{WM}}\), composed of vertices \(\{v_1, \dots,v_n\}\). Pick two intermediate vertices (see Figure \ref{fig:local_vs_global_pf}) \(v_i\) and \(v_j\) such that \(i < j\) and let \(P_{v_i,v_j}\) denote the segment of \(P_{\text{WM}}\) between \(v_i\) and \(v_j\). Now consider solving a subproblem of \eqref{eq:moo_problem} from \(v_i\) to \(v_j\) with the feasible set denoted as \(\mathcal{P}_{v_i,v_j}\). If the cost vector of \(P_{v_i,v_j} \in \mathcal{P}_{v_i,v_j}\) is a point on the convex region of the subproblem's Pareto front, then by Lemma \ref{lem:lemma1} there exists some weight vector such that \(P_{i,j}\) is a solution to the WS problem in \eqref{eq:ws_cost} optimized over \(\mathcal{P}_{v_i,v_j}\). This property is illustrated in Figure \ref{fig:local_vs_global_pf} in which we are minimizing path length and obstacle closeness. Although the cost vector of the global path (blue) lies on the non-convex region of the Pareto front, when considering the subproblem of optimizing from \(v_i\) to \(v_j\), the cost vector of the optimal subpath (green) lies on the convex region of the subproblem's Pareto front. While the existence of this property cannot be guaranteed for an arbitrary graph, we demonstrate in Section \ref{sec:results} that its occurrence in practice is common. 

The core challenge is finding the correct series of break points \(v_i\) and \(v_j\) that define each subproblem and the correct weight vector to apply to each subproblem, which we have shown is also an NP-hard problem. This challenge is inherently combinatorial, which motivates the use of certain optimization frameworks, such as LNS, which have been well-established for their ability to solve such problems.

\section{Large Neighbourhood Search for Weighted Maximization}
In this section, we describe our approach to solving Problem \ref{def:problem2} by using LNS to determine the sequence of WS subproblems and the corresponding weight vectors to use.  

\subsection{WM-LNS Solver Framework}
Given a graph \(G\), start vertex \(s\), end vertex \(g\), and a weight vector \(\textbf{w} \in \mathcal{W}\), the search begins by computing an initial path \(P\). Each iteration then repeats the following steps. First, we uniformly randomly sample a value \(k\) between 1 and \(k_{\max}\) as the length of a continuous segment to remove from our path \(P\) (line 4 in Algorithm \ref{alg:wm‐lns}). Then we probabilistically select some destroy heuristic \(d\) based on its historical success (line 5) and use it to remove a continuous segment of \(k\) vertices from \(P\) leaving an infeasible solution \(P_{\text{partial}}\) with two disjoint subpaths; one from \(s\) to some \(v_i\), and one from \(v_{i+k}\) to \(g\). The disjoint paths are reconnected by solving the WS subproblem between \(v_i\) and \(v_{i+k}\), but using a new weight vector \(\textbf{w}'\) to encourage the exploration of new solutions. We splice the repaired segment back into \(P_{\text{partial}}\), forming a new complete path, \(P_{\text{new}}\). We keep track of the best solution discovered with respect to the input weights \(\textbf{w}\) (line 9), and the new solution \(P_{\text{new}}\) is accepted probabilistically using simulated annealing (line 12) \cite{ropke_adaptive_2006}. The search is terminated after a predefined maximum number of iterations or a fixed number of non‐improving iterations. 

\begin{algorithm}
\caption{WM‐LNS}\label{alg:wm‐lns}
\begin{algorithmic}[1]
\renewcommand{\algorithmicrequire}{\textbf{Input:}}
\renewcommand{\algorithmicensure}{\textbf{Output:}}
\Require Graph $G$, start $s$, goal $g$, weight vector $\mathbf{w}$
\Ensure Path $P_{\text{best}}$ from $s$ to $g$ optimized for $\mathbf{w}$
  \State $P \gets \textsc{initialSolution}(G,s,g,\mathbf{w})$ \Comment{Sec. IV-B}
  \State $P_{\text{best}} \gets P$
  \Repeat
    \State $k \gets \text{Uniformly randomly from }\{1, \dotsc ,k_{\max}\}$
    \State $d = \textsc{selectHeuristic()}$ \Comment{Sec. IV-D, IV-F}
    \State $P_\text{partial} \gets \textsc{destroy}(P,d,k)$ \Comment{Sec. IV-C}
    \State $P_{\text{new}} \gets \textsc{repair}(G, P_{\text{partial}},\mathbf{w'})$ \Comment{Sec. IV-D}
    \If{$\text{WM}(P_{\text{new}}) < \text{WM}(P_{\text{best}})$}
      \State $P_{\text{best}} \gets P_{\text{new}}$
    \EndIf
    \If{\textsc{accept}($P_{\text{new}}$, $P$)} \Comment{Sec. IV-E}
      \State $P \gets P_{\text{new}}$
    \EndIf
    \State \textsc{updateProbabilities()} \Comment{Sec. IV-F}
  \Until{stopping criteria met}
  \State \Return $P_{\text{best}}$
\end{algorithmic}
\end{algorithm}

\subsection{Initial Solution Generation}
\label{sec:initial-sol}
To generate an initial solution, we modify the polynomial-time, suboptimal WM solver variant presented in \cite{wilde_scalarizing_2024} and propose our own variant. The original WM solver described in \cite{wilde_scalarizing_2024} operates by keeping an open list of non-dominated predecessor paths to every node in the graph to discover a Pareto-optimal path. To prevent the open list from growing exponentially, the authors of \cite{wilde_scalarizing_2024} introduce a budget \(b\) on the open list for every node. Instead of placing a hard limit on the number of predecessor paths, we adopt a beam search strategy where the \(b\) best (lowest WM cost) predecessor paths are maintained. Using this proposed modification, we generate an initial solution with a small beam width, which enables the rapid generation of an initial path.

\subsection{Destroy Procedure}
We leverage a set of five destroy heuristics and select one in each iteration to identify which continuous segment \(S \in P\) of length \(k\) to remove. These heuristics aim to diversify the search and target parts of our solution for re-optimization.

\subsubsection{Worst Removal} We remove the continuous segment \(S\) that has the highest WM cost in our path. 

\subsubsection{Best Removal} We remove the continuous segment \(S\) that has the lowest WM cost in our path. While counterintuitive, the aim of this heuristic is to escape local minima by exploring reconnections of locally optimal segments.

\subsubsection{Unbalanced Objective Removal} We remove the continuous segment \(S\) that has the highest mean absolute deviation between objectives.

\subsubsection{Balanced Objective Removal}
We remove the continuous segment \(S\) that has the lowest mean absolute deviation between objectives.

\subsubsection{Random Removal} We remove a random segment \(S\). 

\subsection{Repair Procedure}
Given an infeasible solution \(P_{\text{partial}}\) that contains two disjoint subpaths from \(s\) to some \(v_i\), and from \(v_{i+k}\) to \(g\), we reconnect them by scalarizing the edge costs as a WS and applying \(\text{A}^*\). In Section \ref{subsec:convex-subproblem}, we introduced the property that a subpath's cost vector can lie on a convex region of the subproblem's Pareto front and is thus discoverable by the WS. Given that choosing an optimal weight vector for repair is NP-hard (Problem \ref{def:problem3} in Section \ref{subsec:limitations}), we choose a new weight vector \(\mathbf{w'} \in \mathcal{W}\) through random sampling on a logarithmic scale and obtain some new segment \(S'\) by optimizing the WS problem in \eqref{eq:ws_cost} over the set \(\mathcal{P}_{v_i, v_{i+k}}\). We splice this new segment \(S'\) into \(P_{\text{partial}}\), forming \(P_{\text{new}}\).

\subsubsection{Guided Weight Selection}
Selecting cost-reducing repair weights through random sampling is challenging in higher dimensions; empirically, we found this approach scales poorly beyond a two-dimensional weight space. Therefore, when the dimension of the weight space is greater than two, we turn to derivative-free optimization and apply a modified Generalized Pattern Search (GPS) \cite{audet_derivative-free_2017} in line 7 of Algorithm \ref{alg:wm‐lns} to choose a weight vector for repair.

We begin by randomly sampling \(\mathcal{W}\) on a logarithmic scale to get an initial starting weight, denoted as \(w_{\text{current}}\). In each iteration of our GPS algorithm, we repeat the following steps. First, we generate a set of candidate weights by creating a search mesh, defined by a positive spanning set \(\mathcal{D}\), around the incumbent point. For each direction \(d \in \mathcal{D}\) we randomly sample a step size \(\delta \sim \text{Uniform}(\delta_{\min}, \delta_{\max})\) and form a candidate weight \(w_{\text{candidate}} = w_{\text{current}} + \delta d\) which we project onto the simplex if it falls outside \(\mathcal{W}\) using the algorithm in \cite{condat_fast_2016}. Second, we evaluate each \(w_{\text{candidate}}\) by using it as a repair weight in our WS subroutine and computing \(\text{WM}(P_{\text{candidate}})\). If we find any improvement, we set \(P_{\text{current}}= P_{\text{candidate}}\) and expand our mesh (increase \(\delta_{\min}\) and \(\delta_{\max}\)). If we find no improvement, we shrink the mesh (decrease \(\delta_{\min}\) and \(\delta_{\max}\)).

\subsection{Acceptance Criteria}
To prevent convergence to local minima, we use the simulated annealing criterion from \cite{ropke_adaptive_2006} to determine acceptance of \(P_{new}\). In each iteration, we compute a normalized change in solution cost under the input weights \(\textbf{w}\) using (\ref{eq:wm_cost}):
\[
    \Delta = \frac{\text{WM}(P_{\text{new}}) - \text{WM}(P)}{\text{WM}(P) + \epsilon},
\]
where \(\epsilon > 0\) is a sufficiently small constant. A cost-reducing solution with \(\Delta < 0\) is always accepted, while repeating solutions with \(\Delta = 0\) are skipped. For solutions with \(\Delta > 0\), we accept it with probability \(e^{-\frac{\Delta}{T}}\) where \(T\) is the current temperature. We initialize our temperature to \(T_0\) and decrease it by \(T = c \cdot T\) in every iteration, where \(0 < c < 1\) is the cooling rate. Following \cite{ropke_adaptive_2006}, we initialize \(T_0\) such that a \(p\)\% deterioration is accepted with probability of 50\%. Furthermore, we reheat to \(0.5\cdot T_0\) after a specified number of non-improving iterations. 

\subsection{Adaptive Heuristic Selection}
To adapt toward using the destroy heuristics best suited for the current problem instance, we maintain a set of scores for each heuristic \(s_d\) \cite{ropke_adaptive_2006}. At each iteration, we choose heuristic \(d\) using roulette-wheel sampling with a probability proportional to its performance score. After applying \(d\), we assign it a reward \(\sigma\) (Refer to \cite{ropke_adaptive_2006} and Table \ref{tab:lns_params} for details). Periodically, every \(I\) iterations we compute each heuristic \(d\)'s average reward \(r_{avg} = \frac{\sum reward}{\text{num. uses}}\) and update its score:
\[
    s_d = (1 - \gamma)s_d + \gamma \; r_{\text{avg}},
\]
where \(0 \leq\gamma\ \leq1\) is a specified reaction factor.

\section{Numerical Results}
\label{sec:results}
We conduct a comprehensive series of simulations to assess the performance of WM-LNS. We evaluate WM-LNS on i) the recovery of diverse solutions, ii) computation time, iii) performance across various planning tasks, and iv) its broad applicability by considering a navigation task on a 7-DOF manipulator. We compare WM-LNS against:

\begin{itemize}
    \item \textbf{WM}: We implement the exact version of the WM solver proposed in \cite{wilde_scalarizing_2024} that includes a cost-to-go heuristic.
    \item \textbf{WM-poly}: We implement the polynomial time variant of the WM solver proposed in \cite{wilde_scalarizing_2024}, which places a budget \(b\) on the number of paths leading to every vertex.
    \item \textbf{WM-beam}: We implement our proposed beam search variant of WM-poly from Section \ref{sec:initial-sol}. We allow for a significantly larger budget than was used for the initial solution of WM-LNS.
    \item \textbf{WS}: We replace each edge's cost vector with a scalarized weighted sum cost and run a standard \(\text{A}^*\) search.
\end{itemize}

In each experiment, we tune the budgets of WM-poly and WM-beam to achieve a runtime comparable to WM-LNS. The hyperparameters of WM-LNS are shown in Appendix \ref{sec:hyper-param}.

\subsection{Analysis of Solution Diversity}

\begin{figure}[t]
\centering
\begin{minipage}[t]{0.48\textwidth}
  \centering
  \begin{subfigure}[t]{0.45\textwidth}
    \includegraphics[height=3.0cm]{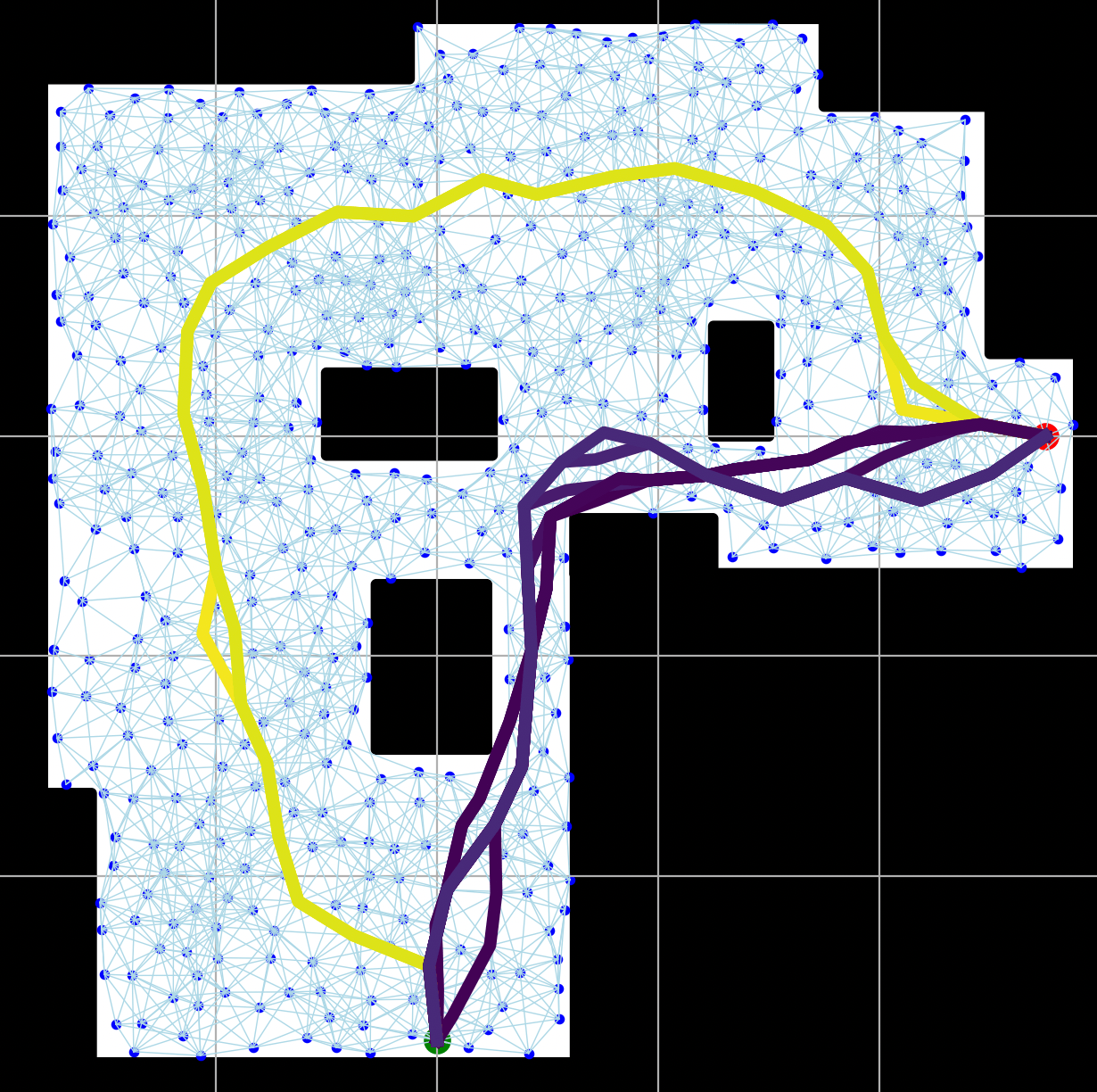}
  \end{subfigure}
  \hfill
  \begin{subfigure}[t]{0.53\textwidth}
    \includegraphics[height=3.1cm]{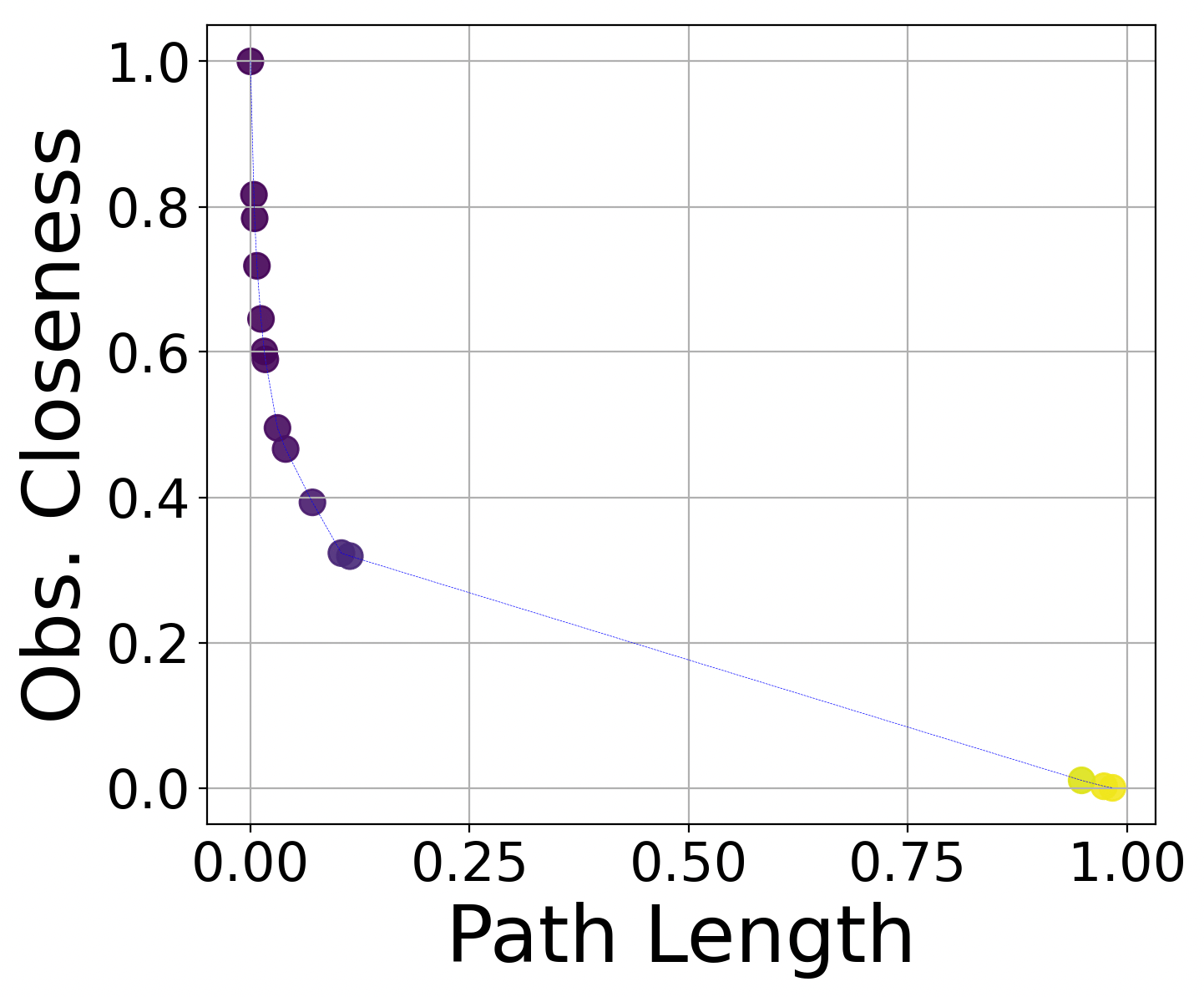}
  \end{subfigure}

  \vspace{0.25em}

  \begin{subfigure}[t]{0.45\textwidth}
    \includegraphics[height=3.0cm]{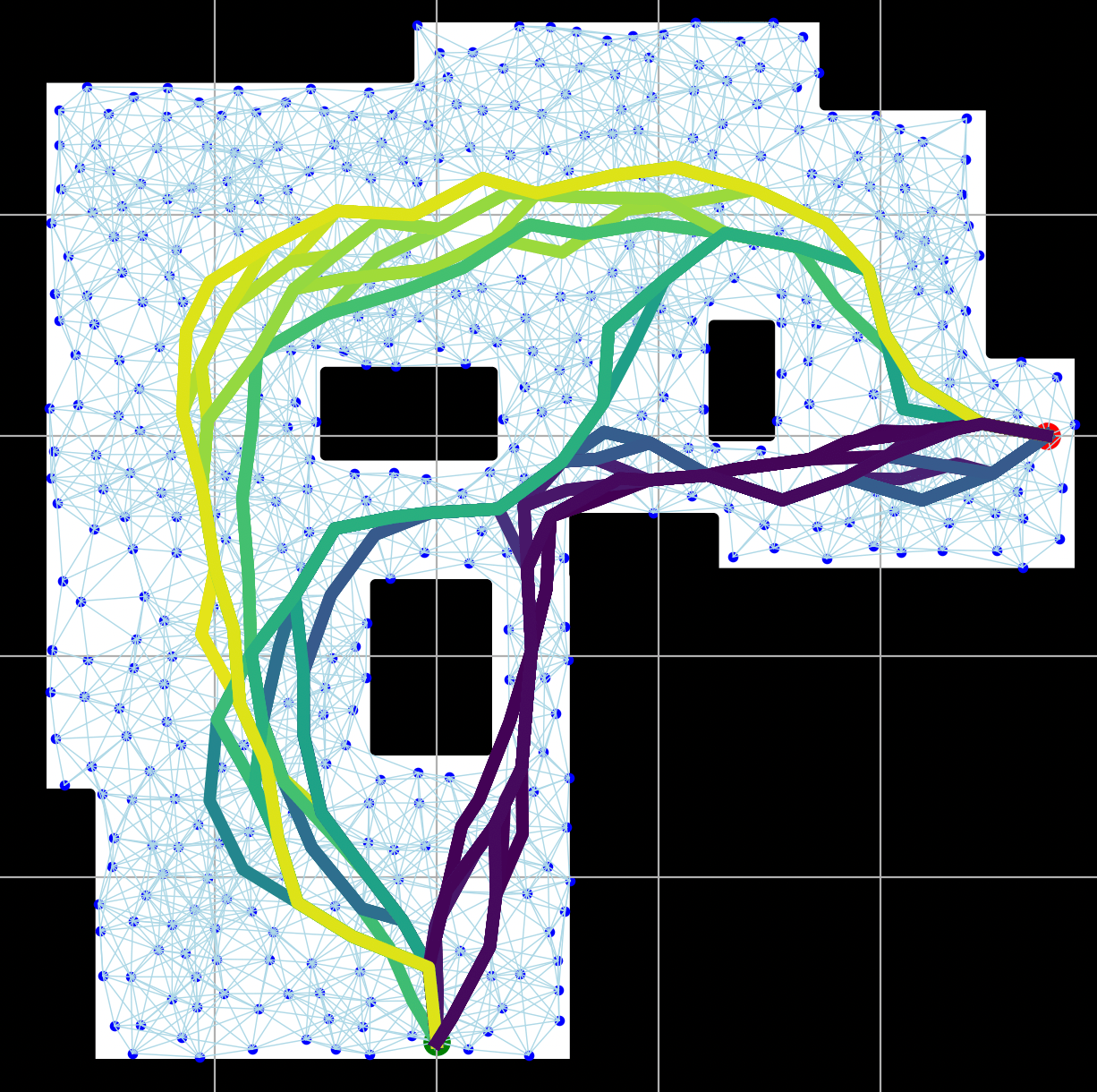}
  \end{subfigure}
  \hfill
  \begin{subfigure}[t]{0.53\textwidth}
    \includegraphics[height=3.1cm]{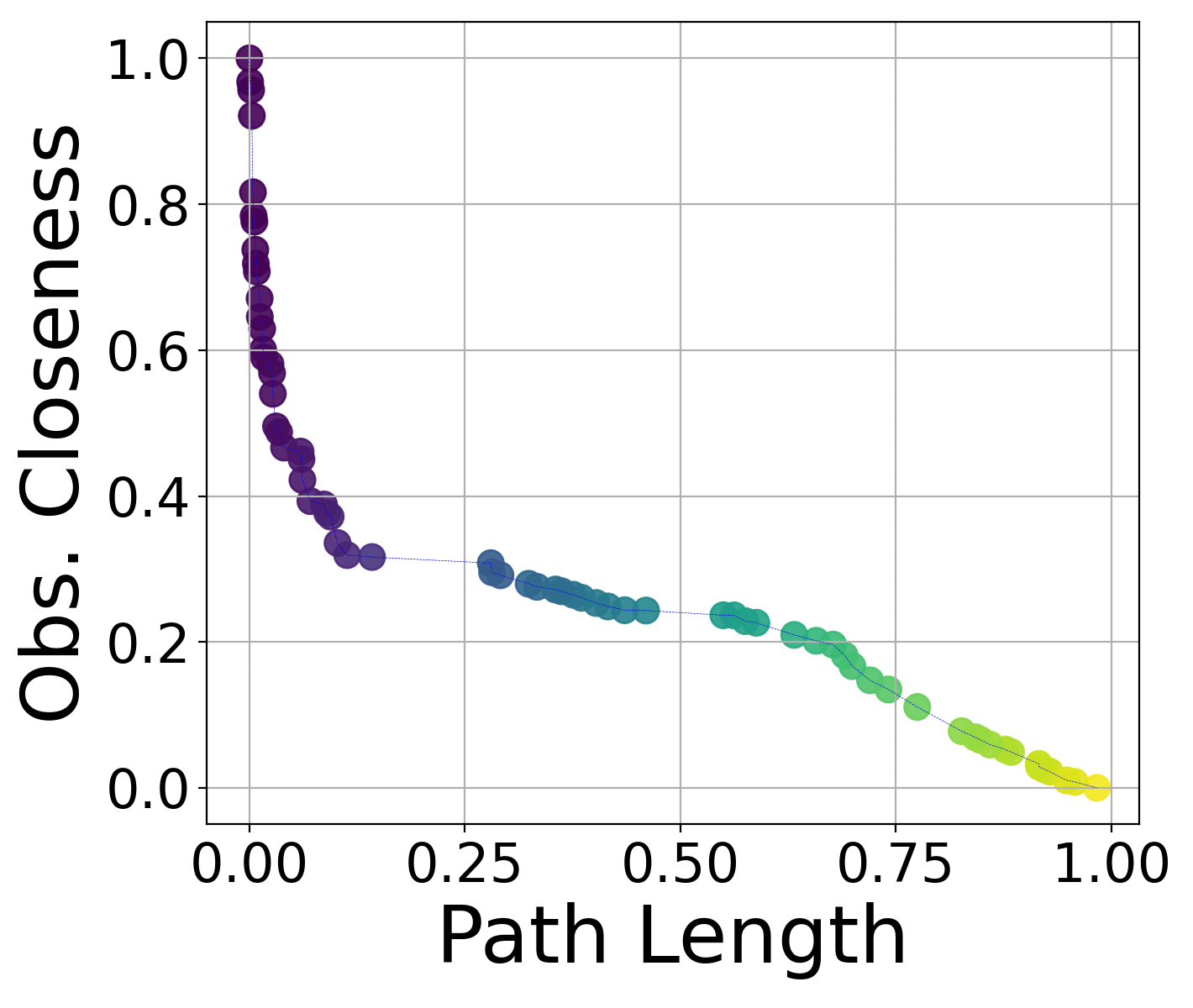}
  \end{subfigure}

  \vspace{0.25em}

  \begin{subfigure}[t]{0.45\textwidth}
    \includegraphics[height=3.0cm]{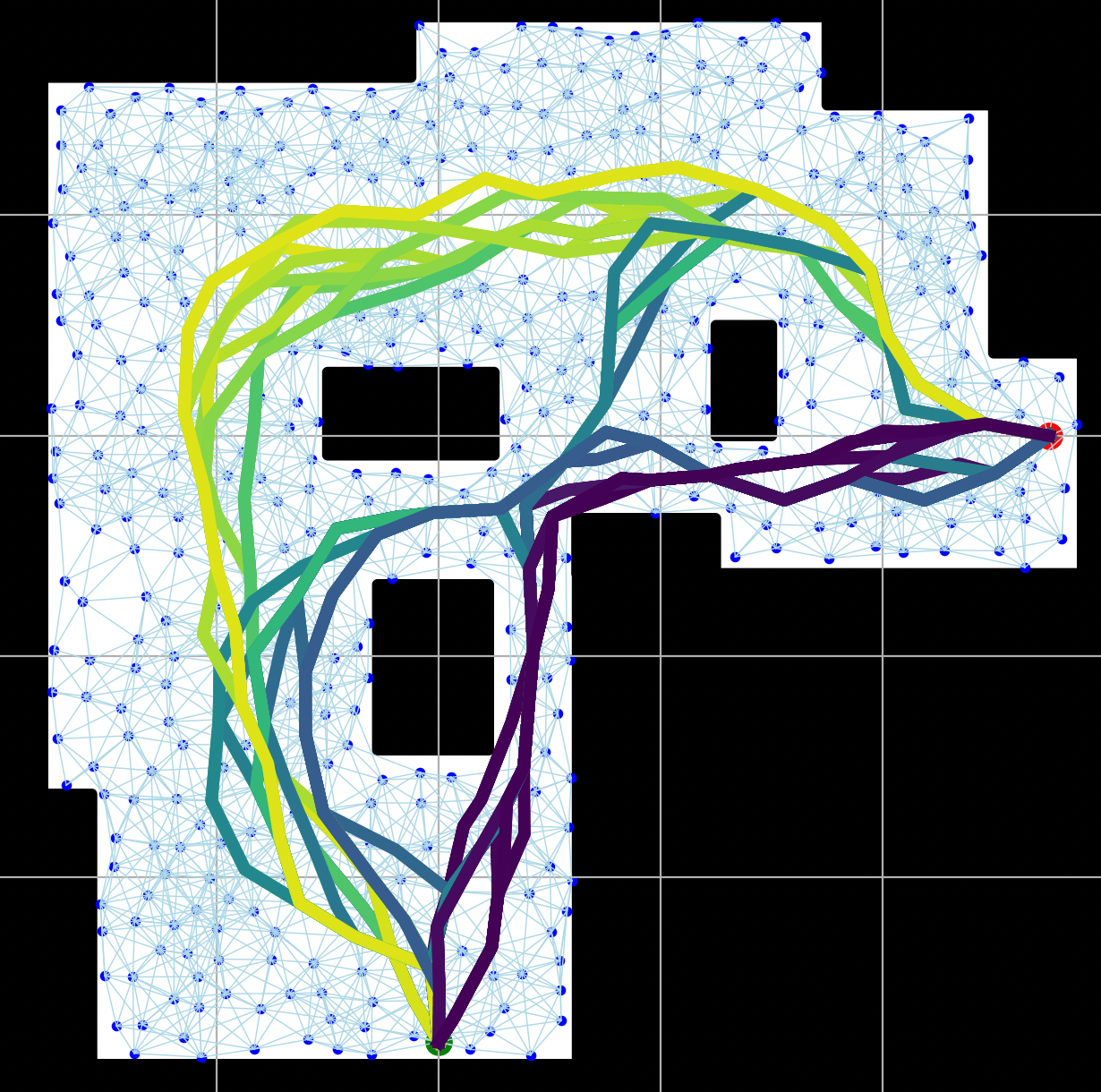}
  \end{subfigure}
  \hfill
  \begin{subfigure}[t]{0.53\textwidth}
    \includegraphics[height=3.1cm]{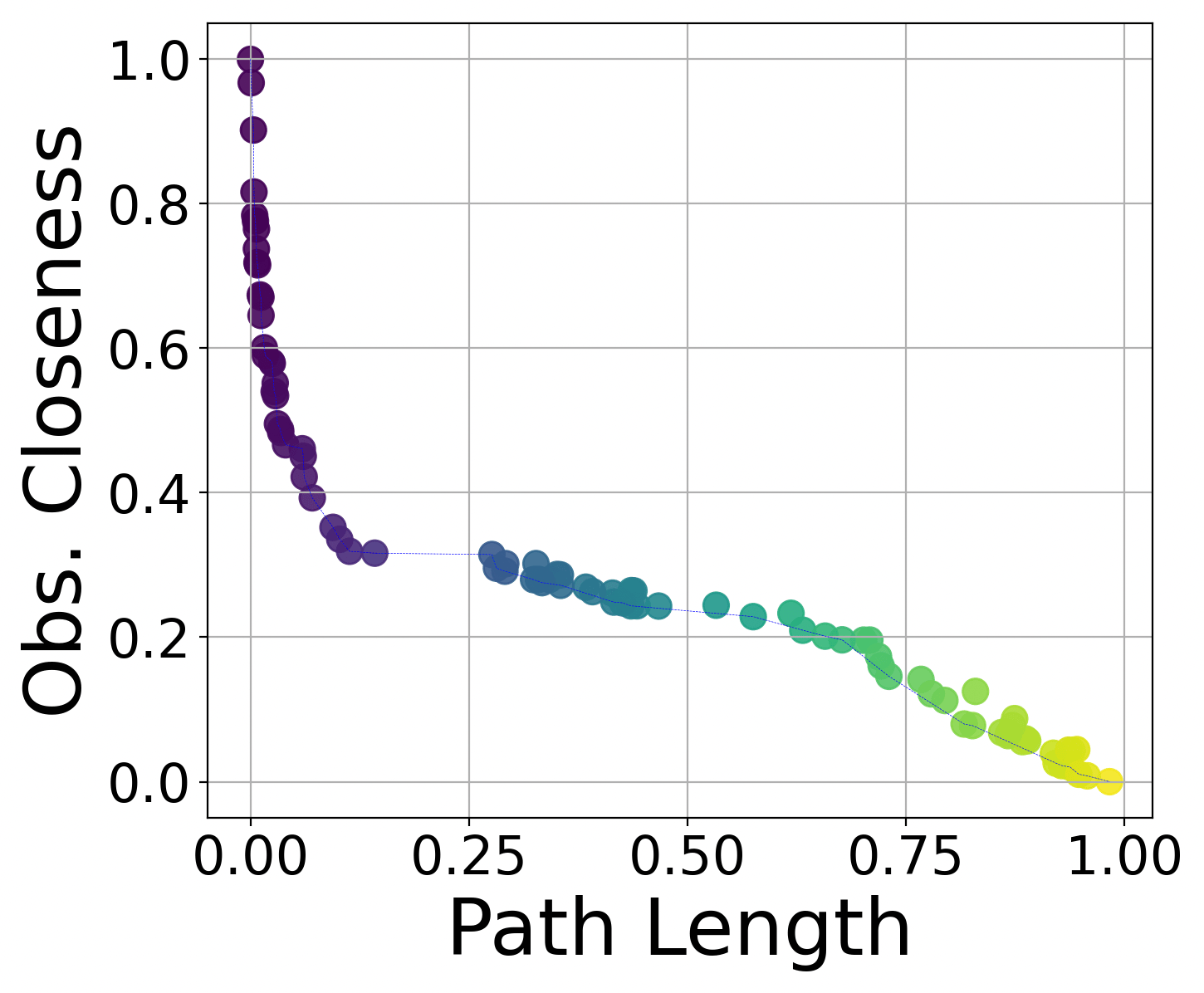}
  \end{subfigure}
\end{minipage}
\hfill
\caption{Solution sets obtained when optimizing two objectives: path length and obstacle closeness. The upper row shows the solution sets found by the WS, the middle row shows the solutions found by the WM, and the bottom row shows the solutions found by the proposed WM-LNS solver. The left column shows the paths found, and the right column shows the corresponding objective values. Note that the WM-LNS plots include suboptimal solutions for illustration; however, Table \ref{tab:solution_set_summary} reports only Pareto-optimal solutions.}
\label{fig:solution_sets_results}
\vspace{-2mm}
\end{figure}

A user may specify \textit{any} weight vector; therefore, we evaluate the ability of WM-LNS to recover diverse solutions across the entire trade-off space. We consider a probabilistic roadmap (PRM) with 500 nodes generated in a cluttered indoor environment, and minimize up to three objectives: path length, obstacle closeness, and risk. Risk is modelled by placing low, medium, and high-risk zones on the base map and adding the corresponding costs to the nodes in each zone (similar to Figure \ref{fig:test_instances}). We conduct two experiments: one for the two-objective case (path length and obstacle closeness) and one for the three-objective case (path length, obstacle closeness, and risk). In each experiment, we conduct 2000 trials, in which we run each solver with randomly sampled weights, record the solution sets obtained, and normalize their objective values. To quantify diversity, we report two metrics. \textit{Coverage} represents the hypervolume of the objective space dominated by a solution set; we sample over the set \([0,1]^n\) to estimate coverage \cite{wilde_scalarizing_2024, 797969}. \textit{Number of Solutions} represents the number of unique Pareto-optimal solutions found as sampling \(n\) different weights does not yield \(n\) different solutions \cite{wilde_scalarizing_2024}.  Results are shown in 
Table \ref{tab:solution_set_summary}.

Using WM-LNS, we can find a diverse solution set with \(13.2\%\) more coverage and \(246.7\%\) more solutions than WS in the two-objective case and \(2.6\%\) more coverage and \(43.2\%\) more solutions in the three-objective case. Compared to the suboptimal WM solvers, most notably WM-beam, we observe that WM-LNS yields \(20-30 \%\) fewer solutions; however, it achieves the same coverage, indicating it finds the same types or homotopy classes of different path solutions. While WM-poly and WM-beam can be used to return more solutions, in Section \ref{sec:runtime-quality}, we demonstrate that their solution quality degrades in larger search spaces. 

For the two-objective case, we plot the resulting solution sets in Figure~\ref{fig:solution_sets_results}. Observe that when using WM-LNS, we find a more diverse set of paths, particularly those with more balanced trade-offs that split the obstacles, a homotopy class of paths that cannot be found by the WS. 

\begin{table}[t]
\centering
\caption{Solution diversity across two and three objectives.}
\label{tab:solution_set_summary}
\setlength{\tabcolsep}{5pt}    
\begin{tabular}{llccc}
\toprule
\textbf{Instance} & \textbf{Method} & \textbf{Coverage $\uparrow$}  & \textbf{\# Solutions $\uparrow$} \\
\midrule
\multirow{3}{*}{Two Obj.} 
& WS & 0.68 & 15 \\
& WM & 0.77 & 69 \\
& WM-poly & 0.76 & 58 \\
& WM-beam (Ours) & 0.77 & 65 \\
& WM-LNS (Ours) & 0.77 & 52 \\
\midrule
\multirow{3}{*}{Three Obj.} 
& WS & 0.77 & 37 \\
& WM & 0.79 & 74 \\
& WM-poly & 0.78 & 56 \\
& WM-beam (Ours) & 0.79 & 73 \\
& WM-LNS (Ours) & 0.79 & 53 \\
\bottomrule
\end{tabular}
\vspace{-1.5mm}
\end{table}

\subsection{Analysis of Runtime \& Solution Quality}
\label{sec:runtime-quality}
The WM solver quickly becomes intractable in large search spaces with many non-dominated paths to explore. As such, our evaluation focuses on the performance in these types of situations. We construct two indoor environments and, for each environment, we build PRMs of increasing size (small, medium, large). For each PRM scale, we perform 50 trials. In each trial, we create an instance by generating a PRM of the specified scale and randomly choosing start and goal nodes. Each instance is solved using weights chosen such that the weighted objective values are approximately equal, since emphasizing a particular objective generally reduces the size of the search space.

\begin{figure}[t]
  \centering

  \begin{subfigure}[t]{\linewidth}
    \centering
    \includegraphics[height=2.5cm,keepaspectratio]{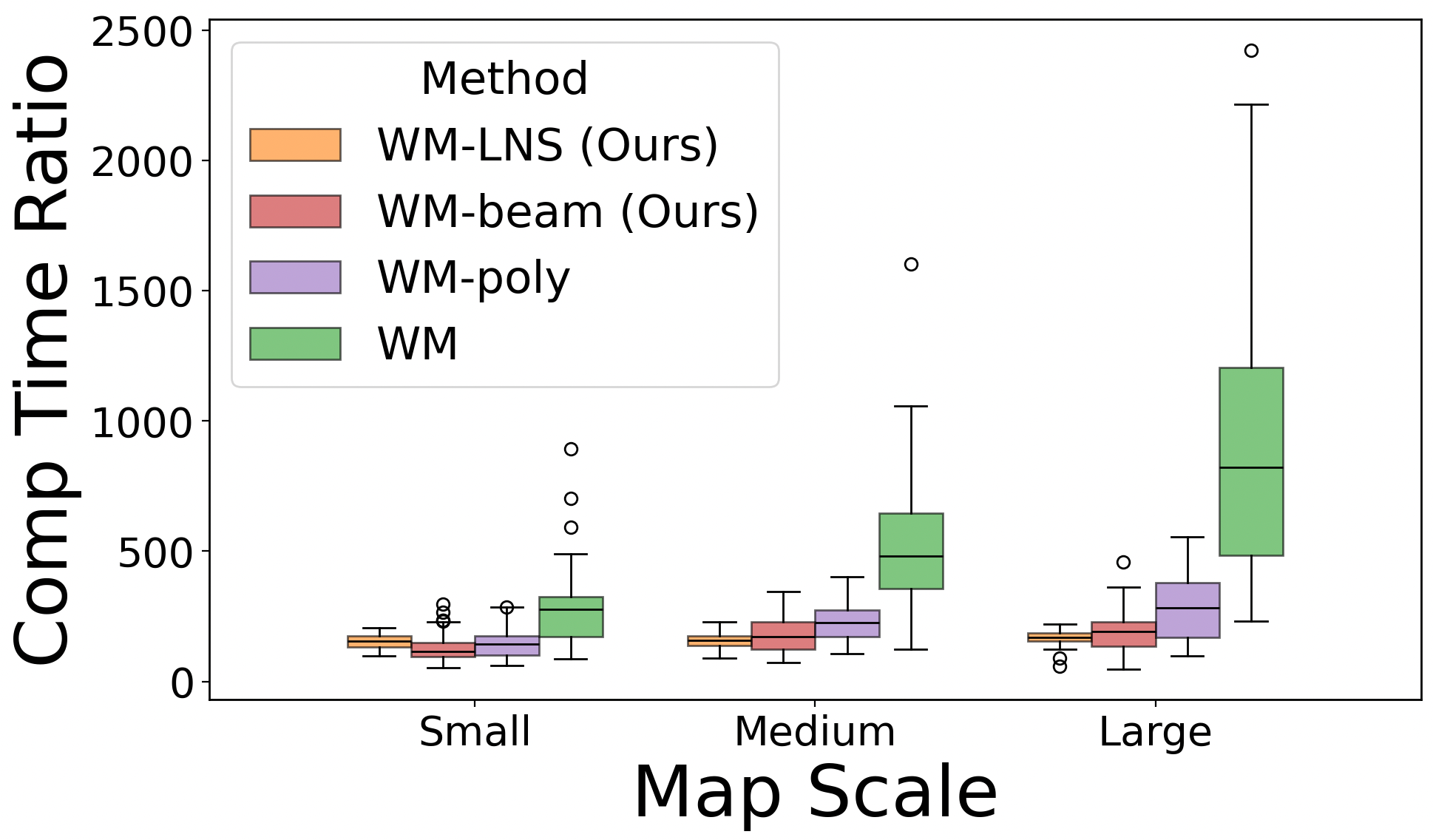}%
    \hfill
    \includegraphics[height=2.5cm,keepaspectratio]{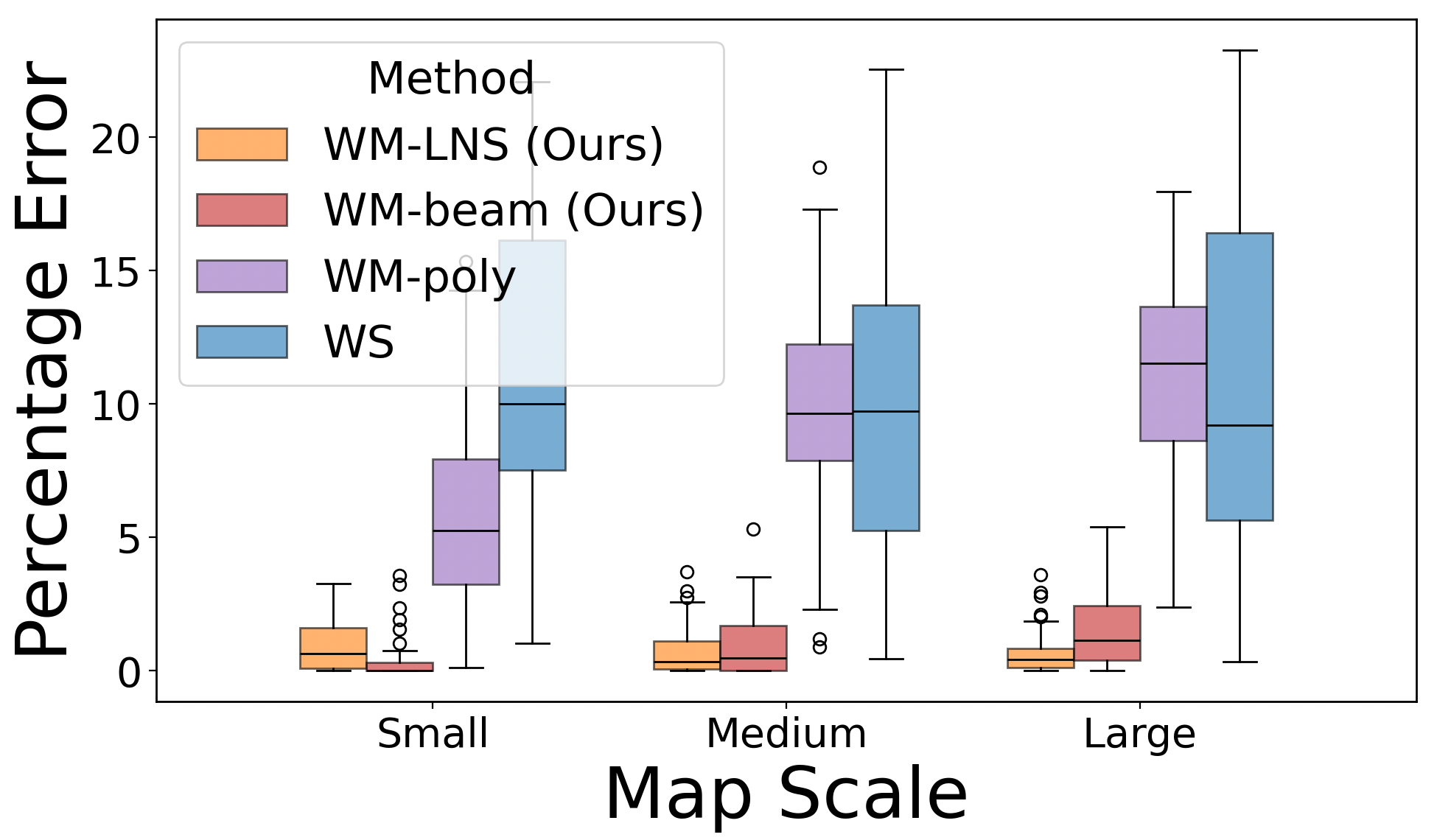}
    \caption{Simulation Results for Instance 1 (Maze).}
    \label{fig:instance-1}
  \end{subfigure}

  \vspace{0.5em}

  \begin{subfigure}[t]{\linewidth}
    \centering
    \includegraphics[height=2.5cm,keepaspectratio]{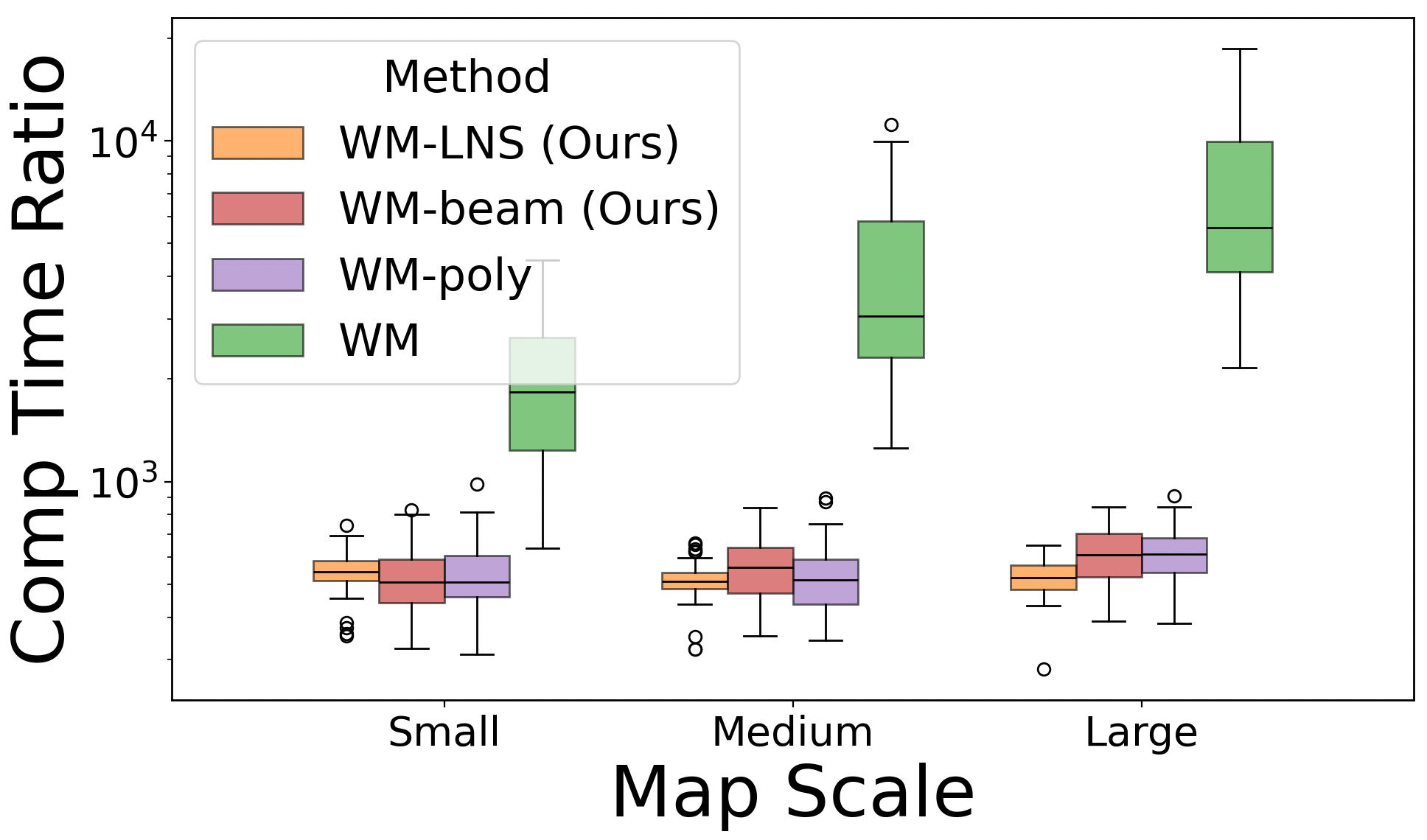}%
    \hfill
    \includegraphics[height=2.5cm,keepaspectratio]{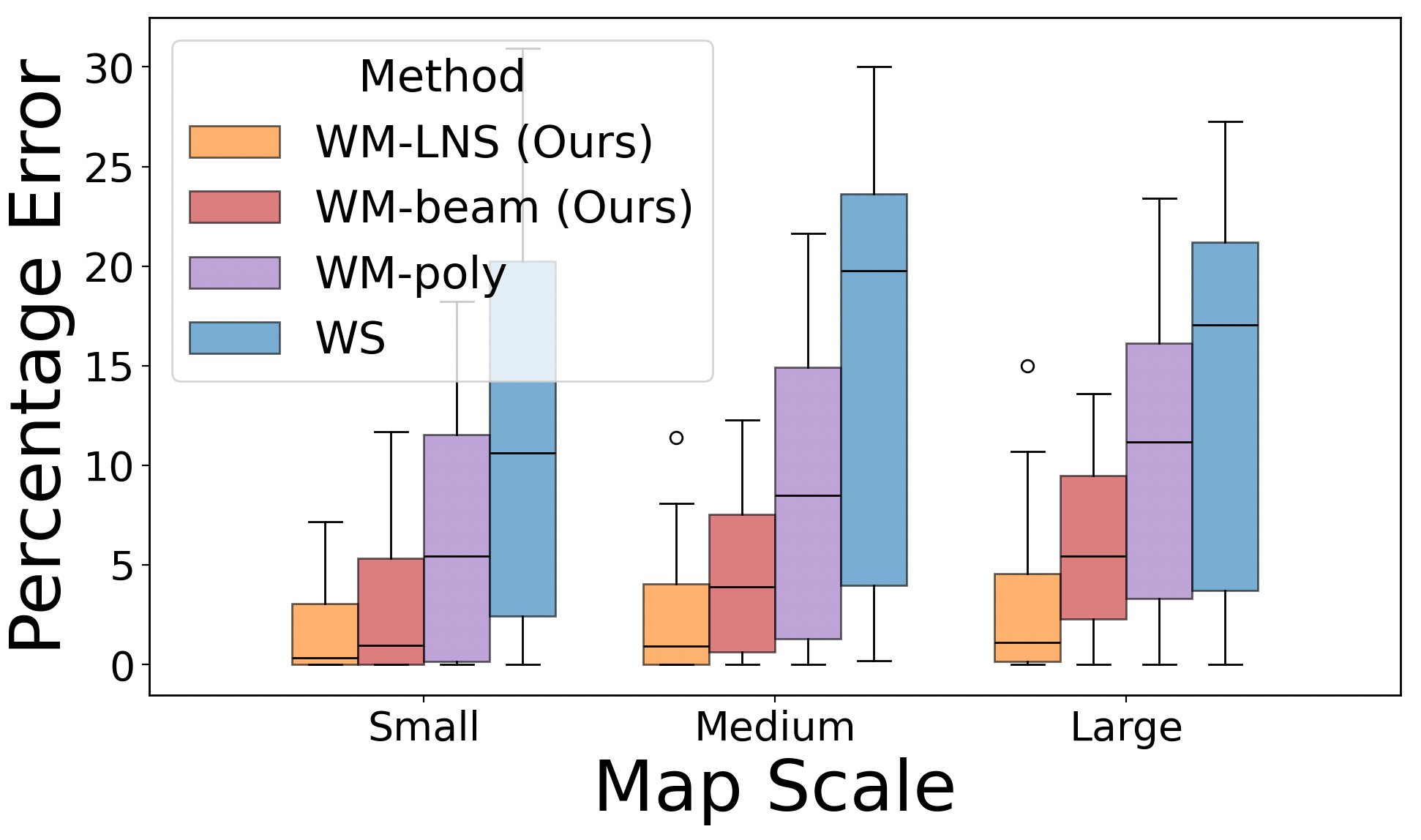}
    \caption{Simulation Results for Instance 2 (Cluttered Boxes).}
    \label{fig:instance-2}
  \end{subfigure}

  \caption{Performance of WM-LNS over two instances. Left plots show the computation time ratio over WS, and right plots show the percentage error from the optimal WM solution.}
  \label{fig:runtime-quality-analysis}
  \vspace{-2mm}
\end{figure}

\paragraph{Map 1 - Maze (Two Objectives)}
In this instance, we consider two objectives: path length and obstacle distance. The PRMs on this map range from 1500 to 2500 vertices. 

On the small map, WM takes on average 3.5 seconds, and the suboptimal and WM-LNS solvers take on average 1.8 seconds; however, on the large map, these times increase to 41.4 seconds and 9.1 seconds, respectively, resulting in a \(4.5\times\) speedup on large maps. In terms of optimality, we observe that WM-LNS and WM-beam are the most competitive, with a mean percentage error ranging from 0.3\% to 1.6\%. However, WM-LNS achieves the lowest percentage error on medium- and large-sized maps, with a mean percentage error of 0.7\% across both sizes.

\paragraph{Map 2 - Cluttered Boxes (Three Objectives)}
Next, we investigate how the performance of each algorithm scales with an increased number of objectives. We consider path length, obstacle distance, and add the objective of minimizing risk. Illustrated in Figure \ref{fig:test_instances}, the red regions correspond to high risk, and the green region corresponds to low risk. The PRMs on this map range from 1000 to 1500 vertices.

The WM solver has an average runtime of 8.7 seconds on small maps, but this increases to  84.5 seconds on the large map, while the average runtime for WM-LNS and the suboptimal variants increases from 2.3 seconds to 6.9 seconds. The increased number of objectives also significantly impacts the solution quality of the suboptimal WM solvers. They rely on a budgeting mechanism to reduce computational effort. However, with more objectives, there exists a larger search space of non-dominated solutions that cannot be captured within smaller budgets. Therefore, to compete with the runtime of WM-LNS, they become significantly suboptimal, whereas WM-LNS remains robust in such situations. Most notably, on large maps, WM-LNS returns solutions \(12 \times\) faster than WM with a mean percentage error of 2.7\%.

\begin{figure}[t]
  \centering

  \begin{subfigure}[t]{0.48\linewidth}
    \centering
    \includegraphics[height=2.9cm,keepaspectratio]{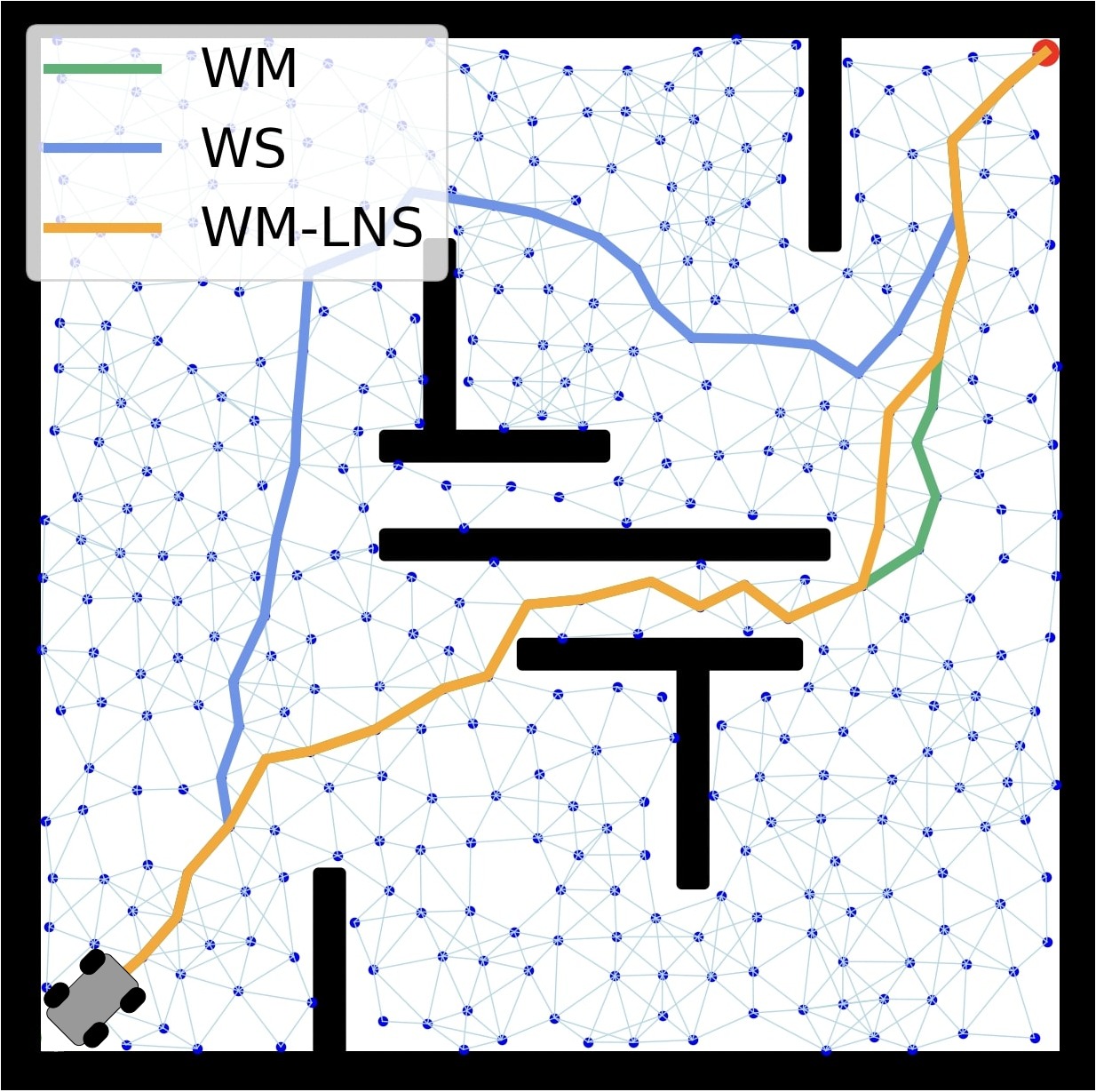}
  \end{subfigure}\hfill
  \begin{subfigure}[t]{0.48\linewidth}
    \centering
    \includegraphics[height=2.9cm,keepaspectratio]{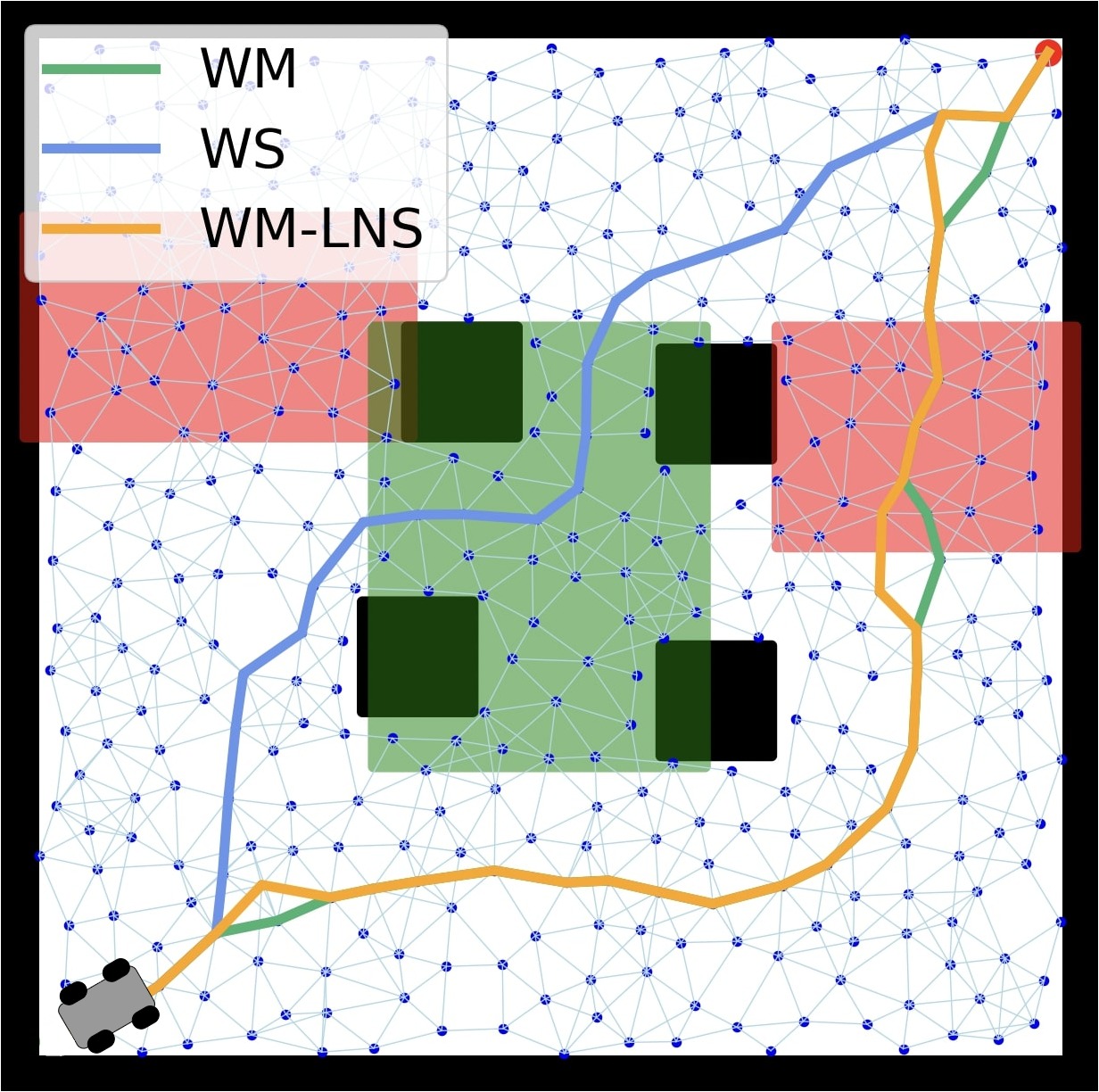}
  \end{subfigure}
  \caption{Test maps used in our experiments. Left: Map 1 - Maze (Two Objectives). Right: Map 2 - Cluttered Boxes (Three Objectives). Red regions correspond to high-risk zones, and green regions correspond to low-risk zones. The paths illustrate solutions for balanced preferences, where weights are chosen such that the weighted objective values are approximately equal.} 
  \label{fig:test_instances}
\end{figure}

\subsection{Planning on a Manipulator}
To demonstrate the broad applicability of the proposed method, we evaluate a multi-objective planning task on a 7-DOF Franka Emika Panda manipulator in the `Franka Kitchen' environment \cite{gupta2019relay}, \cite{kwon_kwonathan/franka-kitchen-pybullet_2025}. We construct a PRM in the robot's task space with 7500 nodes. Over 25 trials, we randomize the start location on the leftmost countertop and the goal location near the sink (see Figure \ref{fig:franka-paths}). The objectives are to minimize path length, proximity to the stove top, and height of the end effector. Similar to the experiments in Section \ref{sec:runtime-quality}, each trial is solved using weights chosen such that the weighted objective values are approximately equal.   

\begin{figure}[t]
  \centering

  \begin{subfigure}[t]{0.48\linewidth}
    \centering
    \includegraphics[height=2.4cm,keepaspectratio]{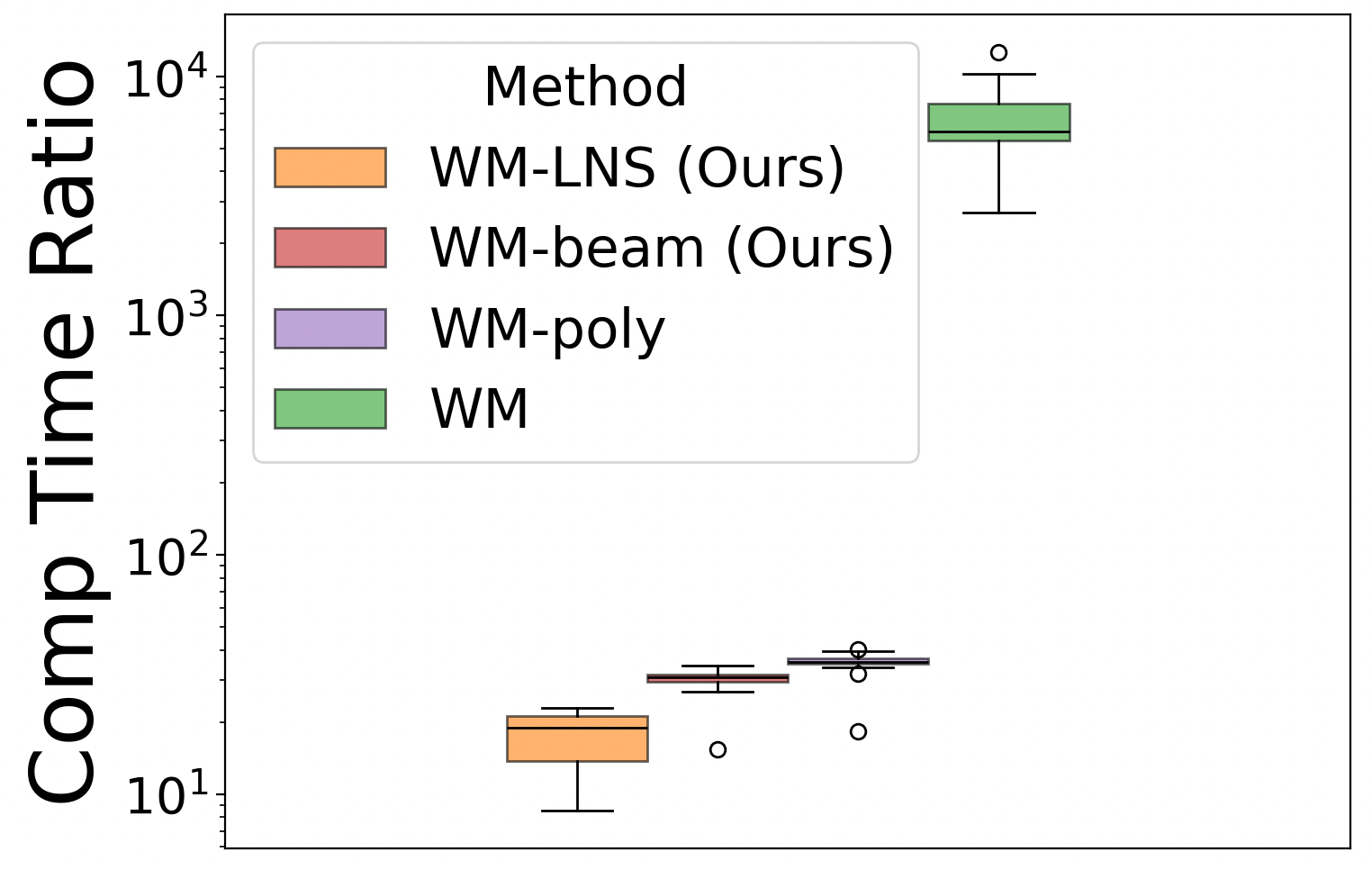}
    \label{fig:franka-comp}
  \end{subfigure}\hfill
  \begin{subfigure}[t]{0.48\linewidth}
    \centering
    \includegraphics[height=2.4cm,keepaspectratio]{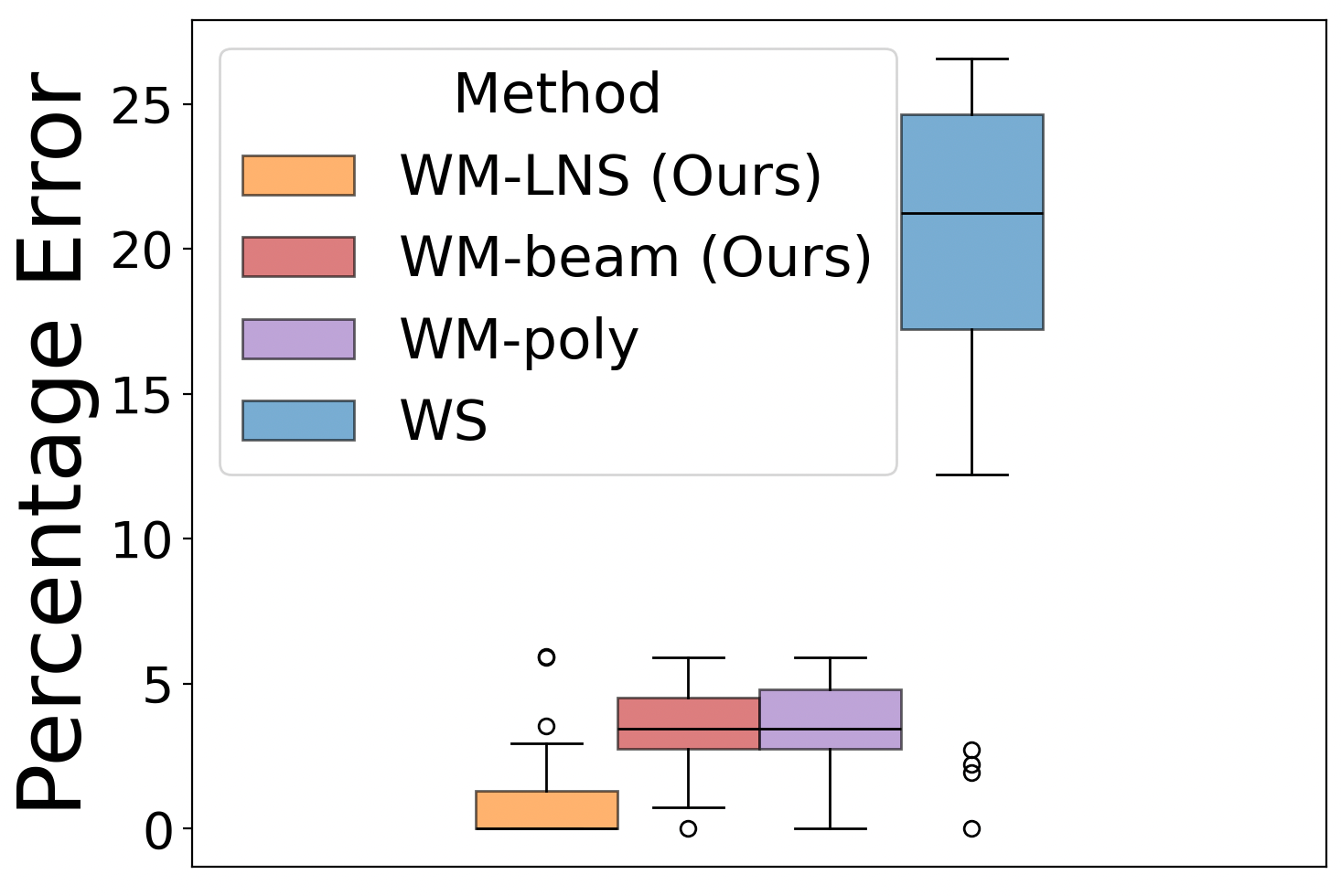}
    \label{fig:franka-cost}
  \end{subfigure}

  \caption{Manipulator planning results. The left plot shows the computation time ratio over WS, and the right plot shows the percentage error from the optimal WM solution.}
  \label{fig:franka-runtime-quality-results}
  \vspace{-2mm}
\end{figure}

WM-LNS averages 2.7 seconds while the WM solver takes 1048.4 seconds, resulting in a \(380\times\) speedup. Across all baselines, WM-LNS achieves the lowest percentage error with a mean value of 1.06\%.

\begin{figure}[t]
    \centering
    \includegraphics[height=3.8cm]{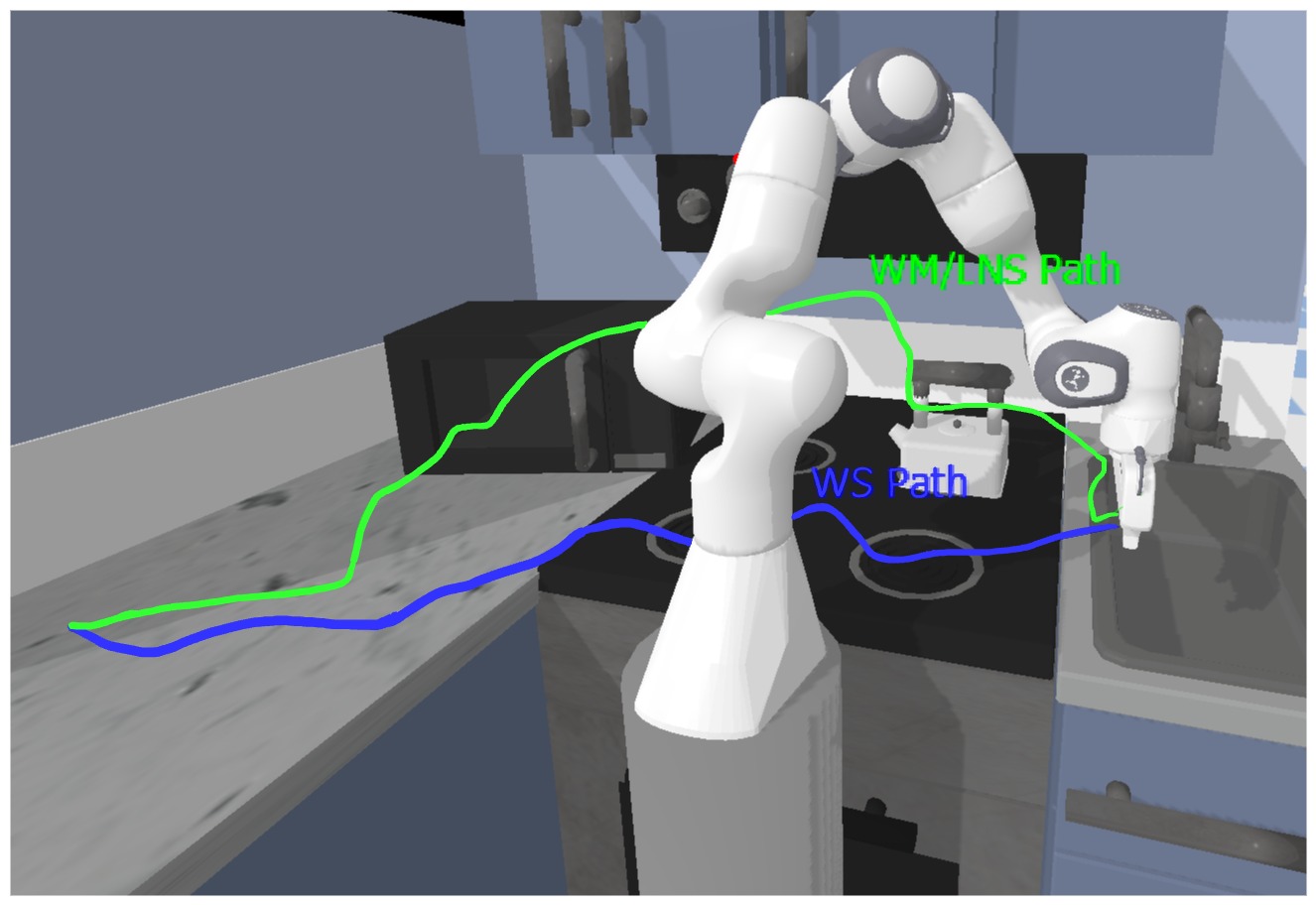}
    \caption{\small Sample results for multi-objective planning in a kitchen environment. The paths illustrate solutions for balanced preferences, where weights are chosen such that the weighted objective values are approximately equal. WM and WM-LNS paths are shown in green (identical), while the WS path is blue.}
    \label{fig:franka-paths}
    \vspace{-2mm}
\end{figure}

\section{Conclusion}
In this work, we studied multi-objective path planning problems scalarized with a WM cost function. We proposed WM-LNS, a heuristic search approach that is capable of generating diverse types of solutions and trade-offs, many of which the widely used WS cannot recover. Across extensive simulations, we also demonstrated that WM-LNS produced near-optimal solutions to the WM problem with a runtime improvement of 1-2 orders of magnitude, outperforming relevant baselines in large search spaces. In future work, we plan to make the selection of neighbourhoods and repair weights more informative, to further improve convergence toward the WM optimum. Another promising direction is to further formalize structural connections between the WM problem and simpler proxy problems (e.g. WS) to guide the design of deterministic approximation algorithms.

\appendices
\section{Hyperparameter Settings}
\label{sec:hyper-param}

\label{app:params}
As WM-LNS operates under a Large Neighbourhood Search framework, we outline the hyperparameters that affect its performance. We initialized the majority of our values from existing literature \cite{ropke_adaptive_2006, audet_derivative-free_2017} and established the remaining initial values empirically. We then performed a local search to select values that yielded the best overall performance. These values were fixed for all experiments except where noted. We use lower iteration limits for GPS-guided (Generalized Pattern Search) sampling and higher iteration limits for random sampling. Manipulator experiments use a beam width of 2. Table \ref{tab:lns_params} lists the full configuration. 

\begin{table}[h]
\caption{WM--LNS Hyperparameters}
\label{tab:lns_params}
\centering
\setlength{\tabcolsep}{6pt}
\begin{tabular}{ll}
\toprule
\textbf{Hyperparameter} & \textbf{Setting} \\
    \midrule
    Max num. iterations & \{75, 400\} \\
    Non-improving Limit & \{25, 50\} \\
    Min. Removed Vertices & \(0.05 \times \text{Path length}\) \\
    Max. Removed Vertices & \(0.95 \times \text{Path length}\) \\
    Cooling Rate \(c\) & 0.985 \\
    Start temp control \(p\%\) & 50\% \\
    Reheat Iteration & \(0.95 \times \text{Non-improving Limit}\) \\
    ALNS window \(I\) & 50 \\
    ALNS Reaction Factor \(\gamma\) & 0.75 \\
    Global Improvement \(\sigma_1\) & 15 \\
    Local Improvement \(\sigma_2\) & 3 \\
    Accepted Deterioration \(\sigma_3\) & 1 \\ 
    Initial Beam Width & \{1, 2\} \\ 
    GPS Iterations & 2 \\ 
    GPS Max Initial Step Size \(\delta_{\max}\) & 0.25 \\
    GPS Min Initial Step Size \(\delta_{\min}\) & 0.125 \\
    GPS Mesh Increase Factor & 2 \\
    GPS Mesh Decrease Factor & 0.25 \\
\bottomrule
\end{tabular}
\end{table}



\end{document}